\documentclass[11pt]{article}

\usepackage[utf8]{inputenc} 
\usepackage[T1]{fontenc}    
\usepackage[margin=1in]{geometry}
\usepackage{times}
\usepackage{microtype}

\usepackage[hidelinks]{hyperref}
\usepackage{url}
\usepackage{doi}

\usepackage{amsmath, amssymb, amsthm, amsfonts, mathtools, nicefrac}
\usepackage{booktabs}
\usepackage{graphicx}
\usepackage{subcaption}
\usepackage{tikz}
\usetikzlibrary{arrows, shapes, positioning}
\usepackage[ruled,vlined]{algorithm2e}

\usepackage{natbib}

\newcommand{\graph}[1]{\ensuremath{\mathcal{#1}}}
\newcommand{\vars}[1]{\ensuremath{\mathbf{#1}}}
\newcommand{\var}[1]{\ensuremath{#1}}

\newcommand\CI{{\,\perp\mkern-12mu\perp\,}}

\newcommand{\hz}{\ensuremath{\mathcal{H}_{\mathbf{z}}}}

\newcommand{\z}{\ensuremath{\mathbf{Z}}}

\newtheorem{assumption}{Assumption}
\newtheorem{definition}{Definition}
\newtheorem{theorem}{Theorem}
\newtheorem{proposition}{Proposition}

\newtheorem{lemma}{Lemma}
\newtheorem{corollary}{Corollary}

\newtheoremstyle{restate}{ }{ }{\itshape}{ }{\bfseries}{.}{ }{\thmname{#1}\thmnote{~#3}}
\theoremstyle{restate}
\newtheorem*{repeatedtheorem}{Theorem}
\newtheorem*{repeatedprop}{Proposition}

\title{\bfseries Transportability without Graphs: A Bayesian Approach to Identifying s-Admissible Backdoor Sets}

\author{
\centering
\begin{minipage}[t]{0.3\textwidth}
\centering
\textbf{Konstantina Lelova} \\
Department of Mathematics and Applied Mathematics \\
University of Crete \\
Greece
\end{minipage}
\hfill
\begin{minipage}[t]{0.3\textwidth}
\centering
\textbf{Gregory F. Cooper} \\
Department of Biomedical Informatics \\
University of Pittsburgh \\
USA
\end{minipage}
\hfill
\begin{minipage}[t]{0.3\textwidth}
\centering
\textbf{Sofia Triantafillou} \\
Department of Mathematics and Applied Mathematics \\
University of Crete \\
Greece
\end{minipage}
}

\date{\today}

\begin{document}
\maketitle

\begin{abstract}
\noindent Transporting causal information across populations is a critical challenge in clinical decision-making. Causal modeling provides criteria for identifiability and transportability, but these require knowledge of the causal graph, which rarely holds in practice. We propose a Bayesian method that combines observational data from the target domain with experimental data from a different domain to identify s-admissible backdoor sets, which enable unbiased estimation of causal effects across populations, without requiring the causal graph. We prove that if such a set exists, we can always find one within the Markov boundary of the outcome, narrowing the search space, and we establish asymptotic convergence guarantees for our method.  We develop a greedy algorithm that reframes transportability as a feature selection problem, selecting conditioning sets that maximize the marginal likelihood of experimental data given observational data. In simulated and semi-synthetic data, our method correctly identifies transportability bias, improves causal effect estimation, and performs favorably against alternatives.
\end{abstract}

\section{Introduction}
\label{sec:intro}
Estimating causal effects is essential for predicting the impact of interventions. Experimental data, such as those from randomized controlled trials (RCTs), provide unbiased estimates but are costly, scarce, and often non-transportable across populations. Observational data, such as electronic health records (EHRs), are abundant but subject to confounding. Increasingly, both experimental and observational data are available: for example, over 80\% of hospitals now maintain basic EHRs \citep{10.1093/jamia/ocx080}, and clinical trial data are widely shared \citep{Sim2022}. \emph{In this work, we propose a method for combining experimental data from a source population with observational data from the target population to test whether a causal effect is both identifiable from observational data and transportable. This allows us to use both data sources for low-variance, unbiased estimation when possible.} 

The method is motivated by the following scenario: We are interested in estimating the post-intervention outcome \var Y given treatment \var X for a target patient population $\Pi^*$, where EHR data  measuring $X, Y$ and a set of pre-treatment covariates $\vars O$ are available. In addition, we have experimental data from an RCT  (e.g., published in the clinical literature) performed on a different population, measuring the same set of variables. The two populations may differ in systematic ways, which can be represented by a (unknown) selection diagram
\citep{pearl_transportability_2011}. The key question is whether we can find a set of covariates $\vars Z\subset\vars O$ such that the causal effect of \var X on \var Y given \vars Z is both identifiable from observational data and transportable across populations. In this case, we can combine the source experimental and target observational data to obtain an unbiased estimate of the post-intervention outcome  in $\Pi^*$. 
Our work makes the following contributions to causal inference from multiple environments:
\begin{itemize}
\item We introduce the first method to estimate the probability that a $\mathbf{Z}$-specific causal effect is simultaneously transportable and identifiable from observational data, without requiring knowledge of the causal graph.
\item We prove that if such a covariate set exists, one can always be found within the Markov boundary of $Y$, thereby reducing the search space.
\item We recast transportability as a feature selection problem. Building on this, we develop a greedy algorithm that explores subsets of the outcome’s Markov boundary to identify the optimal  set $\mathbf{Z}$ for estimating $P(Y \mid do(X), \vars Z, \mathbf{s}^*)$ in the target domain.
\end{itemize}

\begin{figure*}[t!]
    \centering
    \begin{tabular}{cccc}
        \includegraphics[width =0.22\columnwidth]{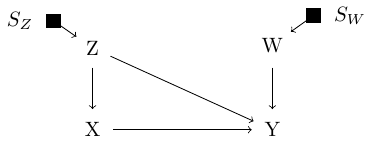}&
        \includegraphics[width =0.2\columnwidth]{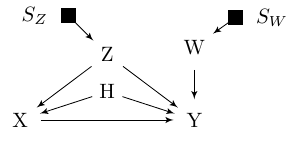}&
        \includegraphics[width =0.2\columnwidth]{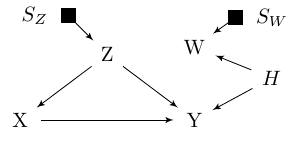}& \includegraphics[width =0.2\columnwidth]{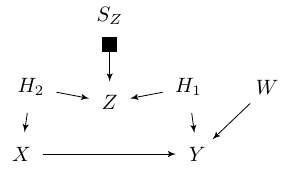}\\
        (a)&  (b)& (c) & (d)\\
    \end{tabular}
    \caption{\label{fig:motivated_example} Causal structures among treatment \( X \), outcome \( Y \), observed pre-treatment covariates \( Z \) and \( W \), and unmeasured covariates \vars H.  (a) $\{Z, W\}$ is an s-admissible backdoor set. (b) No s-admissible backdoor set exists. (c) $\{Z\}$ is an s-admissible backdoor set. (d) $\{W\}$ and $\{ \emptyset\}$ are s-admissible backdoor sets.}
\end{figure*}

The rest of the paper is organized as follows: Section \ref{sec:preliminaries} reviews preliminary concepts related to transportability and identifiability. Section \ref{sec:examples} presents motivating examples that illustrate how our method can make useful inferences. Section \ref{sec:method} details the proposed method. Section \ref{sec:related_work} reviews relevant literature.  Section \ref{sec:experiments} shows that our method makes useful inferences and performs favorably to alternatives in simulated and semi-synthetic data.

\section{Preliminaries}\label{sec:preliminaries}

We adopt the framework of \textit{structural causal models} (SCMs) \citep{pearl_causality_2000}. An SCM $M = \langle \vars U, \vars V, \vars F, P(\vars U) \rangle$ consists of exogenous variables $\vars U$, endogenous variables $\vars V = \{V_1, \dots, V_n\}$, structural functions $\vars F = \{f_1, \dots, f_n\}$ assigning a value to $V_i$ based on a subset of variables in $U \cup (V \setminus V_i)$, and a probability function $P(\vars U)$ defined over the domain of \vars U.  An intervention $do(X=x)$ on $M$ produces a modified model $M_x = \langle \vars U, \vars V, \vars F_x, P(\vars U) \rangle$, where $\vars F_x$ replaces $f_X \in \vars F$ with a constant function returning $x$ for each $X \in \vars X$. The causal Directed Acyclic Graph (DAG), $\graph G$, associated with $M$, induces a distribution $P$ if $P$ factorizes according to $\graph G$. The criterion of $d$-separation can be used on $\graph G$ to determine the conditional (in)dependencies in distribution $P$. $\graph G$ and $P$ are \emph{faithful} to each other if a conditional independence in $P$ implies a $d$-separation in $\graph G$. Interventions $do(X=x)$ are represented by removing incoming edges into $X$, yielding a post-intervention DAG, $\graph G_{\overline X}$, and a post-interventional distribution $P(Y \mid do(X), \mathbf{Z})$, called $\vars Z$-specific causal effect. $\graph G_{\underline X}$ denotes $\graph G$ with edges out of $X$ removed.

\textbf{Selection Diagrams.} Different domains may have different distributions for some variables. Selection diagrams \cite{bareinboim_transportability_2012} use \textit{selection variables}, \(\vars S\) (graphically depicted as square nodes) to represent the mechanisms by which the two domains differ. Formally, let $\langle M, M^* \rangle$ be a pair of structural causal models for domains $\langle \Pi, \Pi^* \rangle$, with corresponding probability distributions $P$ and $P^*$ and a shared causal graph $\graph G=\graph G^*$. The pair is said to induce a selection diagram $\graph D$, if:
 (i) every edge in $\graph G$ or $\graph G^*$ is also in $\graph D$; and
(ii) $\graph D$ includes an extra edge $S_i \rightarrow V_i$ whenever $f_i \neq f_i^*$ or $P(V_i)\neq P^*(V_i)$. We use $P=P(X, Y, \mathbf{O}, |\mathbf{S} = \mathbf{s})$ and $P^*=P(X, Y, \mathbf{O}| \mathbf{S} = \mathbf{s}^*)$ to denote the observational distributions in the source and target domains, respectively. A variable with no incoming selection variable is assumed to have the same generating mechanism in both domains.

\textbf{Identifiability:} The area of identifiability establishes graphical criteria for estimating the post-intervention distribution from the observational distribution \emph{in the same domain}, when the graph is known. For pre-treatment covariates \vars Z,  $P(Y|do(X), \vars Z)$ is identifiable from observational data in the same domain when $\vars Z$ satisfies the backdoor criterion for $X, Y$ in the causal graph $\graph G$ \citep{pearl_causality_2009}:
\begin{theorem}[special case of Rule 2 of do-calculus]\label{the:ident}
Let $\graph G$ be the DAG associated with a causal model, and let $P(\cdot)$ stand for the probability distribution induced by that model. Then if $(Y  \perp\!\!\!\perp  X \mid  \vars Z)_{\graph G_{\underline{X}}}$, the following equation holds:
\begin{equation}
P(Y|do(X), \vars Z) = P(Y|X, \vars Z)
    \label{eq:ident}
    \end{equation}
\end{theorem}
Sets that d-separate \var Y and \var X in  $\graph G_{\underline{X}}$ are called \textbf{back-door sets}. 

\textbf{Transportability.} The area of transportability focuses on graphical criteria for transporting a causal effect \( P(Y| do(X), \vars Z, \vars s) \) in the source domain to a causal effect  \( P(Y | do(X), \vars Z, \vars s^*) \) in the target domain. Sets that make the outcome independent of the selection variables are called \textbf{s-admissible} and are shown to satisfy the following \citep{pearl_transportability_2011}.

\begin{theorem}[S-admissibility]\label{theorem:Theorem 2 JP}
  Let \( D \) be the selection diagram characterizing \( \Pi \) and \( \Pi^* \), and \(\vars S \) the set of selection variables in \( D \). The \( Z \)-specific causal effect \( P(Y | do(X), \vars Z) \) is transportable from \( \Pi \) to \( \Pi^* \) if \( \vars Z \) d-separates \( Y \) from \( \vars S \) in the \( X \)-manipulated version of \( D \), that is, \( \vars Z \) satisfies \( (Y \perp\!\!\!\perp \vars S \mid \vars Z)_{D_{\overline{X}}} \).  A set $\vars Z$ that satisfies this condition is called \textbf{s-admissible}.
  For s-admissible sets, \begin{equation}\label{eq:transp}
P(Y|do(X), \vars Z, \vars s) = P(Y|do(X), \vars Z, \vars s^*)
\end{equation}
\end{theorem} 
s-admissibility stems directly for Rule 1 of do-calculus. 

We are interested in sets that \emph{simultaneously} satisfy both conditions in Theorems \ref{the:ident} and \ref{theorem:Theorem 2 JP}:

\begin{definition}{s-Admissible Backdoor set (\textbf{sABS})}\label{def:s-admissible-backdoor}\\
\noindent A set of variables $\vars Z$ is an \emph{s-admissible back-door set} for $(X,Y)$ relative to a selection diagram $D$ if
(i)  $(Y  \perp\!\!\!\perp  X \mid  \vars Z)_{D_{\underline{X}}}$ and
(ii) $(Y \perp\!\!\!\perp \vars S \mid \vars Z)_{D{_{\overline X}}}$.
\end{definition}

Notice that, since the target and source domain share the same causal graph $\graph G =\graph G^*$, the exact same sets satisfy the backdoor criterion in $\graph G^*$ and \graph D \footnote{The assumption of a shared causal graph can be relaxed without affecting the validity of our method. For brevity, we include this discussion in the Supplementary.}.

The areas of transportability and identifiability provide graphical criteria that allow us to identify and generalize causal effects from observational and/or interventional distributions in different domains, when the graph is known.
However, the causal graph is often unknown and not uniquely identifiable from available data. In the next sections, we propose a method for testing if these two criteria jointly hold when the graph is unknown.

\section{Motivating Examples}\label{sec:examples}
We now motivate our method with examples. Our goal is to estimate the post-intervention distribution of the outcome $Y|do(X)$ in a target domain $\Pi^*$, using pre-treatment covariates $\vars O$. We assume the following:

\begin{assumption}\label{ass:main}
 Let \graph D be a selection diagram and $P(X, Y, \vars O, \vars S)$ a distribution faithful to \graph D, which is strictly positive. Let $D_o^*$ be an observational dataset with $N_o$ samples of $X, Y, \vars O$ in the target domain $\Pi^*$, sampled from distribution $P^* = P(X, Y, \vars O|\vars S=\vars s^*)$. Let $D_e$ be an experimental dataset with $N_e$ samples measuring the same variables in domain $\Pi$, sampled from distribution $P_{\overline X} = P(do(X), Y, \vars O|\vars S=\vars s)$.      
\end{assumption}

We propose to leverage both $D_o^*$ and $D_e$ to find covariate sets $\vars Z \subseteq \vars O$ that are \emph{s-admissible backdoor sets} with respect to $\graph D$. When such a set exists, the $\vars Z$-specific causal effect can be consistently estimated by combining observational and experimental data and applied to the target domain. Importantly, the s-admissibility and the backdoor criterion may hold for a subset of the observed covariates, but not for the full set.  

Figure~\ref{fig:motivated_example} illustrates four representative cases:\\
\textbf{Example 1}, Fig.~\ref{fig:motivated_example}(a):  
$S_Z$ and $S_W$ encode distributional shifts between domains. The set $\{Z, W\}$ is an s-admissible backdoor set, so both $D_e$ and $D_o^*$ can be used to estimate $P(Y|do(X), Z, W, \vars s^*)$.

\textbf{Example 2}, Fig.~\ref{fig:motivated_example}(b): \var H is unobserved, so no s-admissible backdoor set exists.   However, $\{Z, W\}$ is s-admissible (but not backdoor), hence $D_e$  yields an unbiased estimator for $P(Y|do(X), Z, W, \vars s^*)$. However, without knowing the causal graph, we cannot know that $\{Z, W\}$ is s-admissible.

\textbf{Example 3}, Fig. \ref{fig:motivated_example}(c): $\{Z\}$ alone is an s-admissible backdoor set, so both $D_e$ and $D_o^*$ can be used to estimate $P(Y|do(X), Z, \vars s^*)$. Notice that $\{Z, W\}$ is a backdoor set (but not s-admissible), so $D_o$ allows for unbiased estimation of $P(Y|do(X), Z, W,\vars s^*)$. However, without knowing the causal graph, we cannot know that $\{Z, W\}$ is a backdoor set.

\textbf{Example 4}, Fig. \ref{fig:motivated_example}(d): $\{W\}$ alone is an s-admissible backdoor set, so both $D_e$ and $D_o^*$ can be used to estimate $P(Y|do(X), Z, \vars s^*)$.  The set $\{Z, W\}$ is neither s-admissible, nor a backdoor set, so estimating  $P(Y|do(X), Z, W, \vars s^*)$ using any available data leads to biased estimations.

Our proposed method uses the available data to identify an s-admissible backdoor set, if one exists (otherwise, it returns $NaN$). Hence, our method would return  $\{Z, W\}$ in Example 1,  $NaN$ in Example 2,  $\{Z\}$ in Example 3, and  $\{W\}$ in Example 4, and the corresponding estimators (when not $NaN$) based on $D_e, D_o^*$.  While alternative estimators based on only $D_e$ or only $D_o^*$ may be better (e.g., include additional informative covariates),  their unbiasedness cannot be guaranteed without knowledge of the causal graph. Restricting to sABS's ensures unbiased estimation.

\section{Method}\label{sec:method}
The underlying idea of our proposed method is the following: If a set of variables  \vars Z  is an \textbf{sABS} for $X, Y$ with respect to \graph D, the following holds:

\begin{proposition}\label{prop:sabs}
  Let \( D \) be the selection diagram characterizing \( \Pi \) and \( \Pi^* \), and let \(\vars S \) be the set of selection variables in \( D \). If \( Z \) is an s-admissible backdoor set for $X, Y$ relative to \graph D, then \begin{equation}\label{eq:sabs1}P(Y|do(X), \vars Z, \vars s) = P(Y|X,\vars Z, \vars s^*)
  \end{equation}
\end{proposition}

The proof can be found in the supplementary, and is a direct consequence of Eqs.~\ref{eq:ident}, \ref{eq:transp}, which hold simultaneously for s-admissible backdoor sets. Hence, if $\vars Z$ is an sABS, the conditional distribution of the outcome in the target observational distribution is the same as in the source experimental distribution. Moreover, in this case, our target estimand $P(Y|do(X), \vars Z, \vars s^*)$ coincides with both sides of Eq.~\ref{eq:sabs1}. Thus, all available data can be used to predict $\var{Y|do(X)}$ in the target domain.

Notice that, for most faithful distributions, Eq. \ref{eq:sabs1} will \emph{not} hold if \vars Z is \emph{not} an sABS. However, there are cases where, through some accidental parameter choices,  the confounding bias in the target domain and the transportability bias cancel each other out. If such a cancellation were to occur, Eq.~\ref{eq:sabs1} would hold, even though \vars Z is neither s-admissible, nor a backdoor set. However, this would reflect an accidental parameter alignment rather than a structural property of the system, and the target estimand would not be equal to either side of Eq. \ref{eq:sabs1}. Assumption~\ref{ass:sfaith} rules out coincidences, in the same spirit that ordinary faithfulness rules out accidental independencies:
\begin{assumption}\label{ass:sfaith}[sABS-faithfulness]  We assume that $P(X, Y, \vars O, \vars S)$  is faithful to  \graph D. Moreover, if \vars Z is not  s-admissible for $X, Y$ in \graph D, and \var Z is not a backdoor set in \graph D, then 
\[\exists\;x, y, \vars z\textnormal{ s.t. } \frac{P(y|do(x), \vars z, \vars s)-P(y|do(x), \vars z, \vars s^*)}{ P(y|x, \vars z, \vars s^*)-P(y|do(x), \vars z, \vars s^*)}\neq 1\]
\end{assumption}
It is easy to show that if a distribution is sABS - faithful to \graph D, then Eq. \ref{eq:sabs1} only holds if \vars Z is an sABS (A formal proof can be found in the Supplementary). Hence, Proposition \ref{prop:sabs} and Assumption \ref{ass:sfaith} allow us to test if \vars Z is an sABS, by checking if Eq. \ref{eq:sabs1} holds in our available data. If \vars Z is an sABS, we can use $D_e$ and $D_o^*$ to obtain an unbiased estimator of $P(Y|do(X), \vars Z, \vars s^*)$. If $\vars Z$ is not an sABS, we cannot be sure that such a $\vars Z$-specific unbiased estimator exists in our available data, so we return no estimator. In the next section, we propose: (a) a method for computing the probability that a set is an sABS, and (b) a search strategy for identifying an optimal sABS, if one exists.

\subsection{Probability that a set is an sABS}
To estimate the probability that a set $\vars Z$ is an sABS, we introduce a binary variable \hz, with $\hz =h_{\vars Z}$ if \vars Z is an sABS, and $\hz =\neg h_{\vars Z}$ if it is not.  Under $h_{\vars Z}$, $\vars Z$ is an s-admissible backdoor set in \graph D. Therefore, Eq. \ref{eq:sabs1} holds. In contrast, under $\neg h_{\vars Z}$, Eq. \ref{eq:sabs1} does not hold. $P(\mathcal{H}_\vars{Z} =h_{\vars Z}|D_e, D_o^*)$ can be computed on the basis of this observation, to reflect the compatibility between the source experimental and target observational data. Following the approach in \cite{pmlr-v206-triantafillou23a}, we can use the Bayes rule to obtain the following Equation:
\begin{equation}\label{eq:hzt}
P(h_{\vars Z}|D_e,D_o^*) =
\frac{ P(D_e|h_{\vars Z},D_o^*)P(h_{\vars Z}|D_o^*)}{\sum\limits_{\hz\in\{ h_{\vars Z}, \neg h_{\vars Z}\}}P(D_e|\hz, D_o^*)P(\hz|D_o^*)}
\end{equation}
The heart of Eq.\ref{eq:hzt} is the marginal likelihood $P(D_e|h_{\vars Z},D_o^*)$, which quantifies how well we can predict the outcome in the source experimental data, given the target observational data, as justified by Proposition \ref{prop:sabs}. The terms in Eq. \ref{eq:hzt} are discussed below.

\textbf{Estimating $P( H_{\vars Z}|D_o^*)$}. This quantifies the probability that $\vars Z$ is an sABS, given only the target observational data $D_o^*$. This can be viewed as a prior for \hz given just the target observational data. In our case, the observational data do not carry enough information for the hypothesis \hz, since observational data from a single domain cannot be used to determine whether $\vars Z$ is s-admissible. For this reason, we use the  uninformative distribution $P(h_{\vars Z}|D_o^*) = P(\neg h_{\vars Z}|D_o^*)=0.5$, indicating that, given only observational data in the target domain,  each set is plausibly an sABS. The relative effect of $P(H_{\vars Z}|D_o^*)$ on Eq. \ref{eq:hzt} is small, as it remains constant regardless of the size of the experimental data. An ablation study showing that this effect is negligible even for small experimental sample sizes can be found in the Supplementary. A similar result was shown in \citep{pmlr-v206-triantafillou23a}.


\begin{algorithm}[t]
\LinesNumbered
\SetKwInOut{Input}{input}
\SetKwInOut{Output}{output}
\Input{$X,Y,\vars Z,D_{o}^*,D_{e}$, MCMC samples \var N}
\Output{$P(D_e|D_o^*, H_{\vars Z}), P(H_{\vars Z}\mid D_e, D_o^*)$}
\ForEach{$i =1, \dots,N$}{
Sample $\theta_e^i$ from an un-informative prior $f(\theta_e)$\\
Compute likelihood $\graph L_0 = P(D_e | \theta_{e}^i)$\\
Sample $(\theta_o^*)^i$ from the observational posterior $f(\theta_o^*|D_o^*)$ using MCMC\\
Compute likelihood $\graph L_1(i) = P(D_e | (\theta_o^*)^i)$}
$P(D_e|D_o^*, \neg h_{\vars Z}) \leftarrow \frac{1}{N}\sum_i \graph L_0(i)$\\
$P(D_e|D_o^*, h_{\vars Z})\leftarrow\frac{1}{N}\sum_i \graph L_1(i)$\\
Compute $P(h_{\vars Z}|D_e, D_o^*)$ using Eq.\ref{eq:hzt}
\caption{ProbsABS}\label{algo:probsabs}
\end{algorithm}

\textbf{Estimating $P(D_e|D_o^*, H_{\vars Z})$}
This represents how likely the source experimental data are, given the target observational data, and the fact that \vars Z is (not) an sABS. Under $h_{\vars Z}$,  $P(Y|do(X), \vars Z, \vars s)=P(Y|X,\vars Z, \vars s^*)$, hence, the source experimental and target observational distributions are the same. Let  $\theta_e, \theta_o^*$ denote the parameters of the conditional distributions $P(Y|do(X), \vars Z, \vars s), P(Y|X,\vars Z, \vars s^*)$, respectively. $P(D_e|D_o^*,\hz)$ can be computed as the marginal likelihood of a model predicting the post-intervention $Y$ in the source domain, using observational data as a prior, under the two competing values of \hz:
\begin{equation}\label{eq:exp_score}
     P(D_e|\hz, D_o^*) = \int_{\mathbf{\theta_e} }P(D_e|\mathbf{\theta_e})f(\mathbf{\theta_e}|D_o^*, \hz)d\mathbf{\theta_e}
\end{equation}
Under $\hz =h_{\vars Z}$, $\mathbf{\theta_e} = \mathbf{\theta_o^*}
$, therefore $f(\theta_e|D_o^*, h_{\vars Z}) = f(\theta_o^*|D_o^*)$. Thus, Eq. 
\ref{eq:exp_score} can be rewritten using observational parameters, and computed in closed form for distributions with conjugate priors, or approximated using sampling.

Under $\hz =\neg h_{\vars Z}$, $\theta_e \neq \theta_o^*$, the target observational distribution is not informative  (at least for point estimation) for the source experimental distribution, and therefore $f(\theta_e|D_o^*,\neg h_{\vars Z}) = f(\theta_e)$. Eq \ref{eq:exp_score} is then the marginal likelihood of a model with uninformative priors, and can again be computed in closed form or approximated with sampling. We note that this Bayesian formulation naturally accommodates sample size imbalance between observational and experimental data, as the larger observational dataset informs the prior for the typically smaller experimental sample. Algorithm \ref{algo:probsabs} describes a sampling-based approximation method. If a closed-form solution is feasible, Lines 1-5 can be skipped, and the closed-form equations can be used in Line 6. Specific equations for mixed and discrete data and prior choices can be found in the Supplementary. 

\begin{algorithm}[t!]
\LinesNumbered
\SetKwInOut{Input}{input}
\SetKwInOut{Output}{output}
\SetKwFunction{MarkovBoundary}{MarkovBoundary}
\SetKwFunction{ProbsABS}{ProbsABS}
\SetKwFunction{Score}{Score}
\Input{$X,Y,\vars O$, $D_{e}, D_o^*$, sampling iterations \var {N_s}, probability threshold \vars t}
\Output{sABS $\vars Z^*$,  $P(Y|do(X), \vars Z^*, \vars s^*)$}
{$\Score( \vars Z) \coloneqq \ProbsABS(X, Y, \vars Z, D_o^*, D_e, N_s)$}\;

MB(Y) $\leftarrow \texttt{ MarkovBoundary}(Y, D_o^*)$\;
$P(Y|do(X), \vars Z^*, \vars s^*)\leftarrow$NaN\;
$\vars Z\leftarrow \emptyset$, cur\_score$\leftarrow \Score( \emptyset)$, \emph{found}=false\;

\While{found ==false}{
\ForEach{ $\var Z\in \vars MB(Y)\setminus \vars Z$ } {
\begin{small}
$P(D_e|D_o^*, h_{\vars Z\cup Z})\leftarrow\Score(\vars Z\cup Z$)\;
\end{small}}
\ForEach{ $\var Z\in \vars Z$ }{
\begin{small}
$P(D_e|D_o^*, h_{\vars Z\setminus Z})\leftarrow\Score(\vars Z\setminus Z$)\;
\end{small}}


$\vars Z^* \gets \arg\max_{\vars Z' \in \{\;\vars Z\cup Z,\; \vars Z\setminus Z\}} 
   P(D_e|D_o^*, h_{\vars Z'})$\;

\If {$P(D_e|D_o^*, h_{\vars Z^*})>\textnormal{cur\_score}$}{
$\vars Z\leftarrow \vars {Z^*}$, $\textnormal{cur\_score} \leftarrow P(D_e|D_o^*, h_{\vars Z^*})$\;
}
\Else{
\emph{found}==true, $\vars Z^*\leftarrow \vars Z$\;
}
}
\If {$P(h_{\vars Z^*}|D_e, D_o^*)>t$}{$P(Y|do(X), \vars Z^*, \vars s^*) \leftarrow P(Y|do(X), \vars Z^*, D_e, D_o^*)$}
\caption{FindsABS}\label{algo:FindsABS}
\end{algorithm}

\subsection{Finding an sABS}
Using Eq. \ref{eq:hzt}, we can compute the probability that any set \vars Z is an sABS. However, what we ultimately want is to compute the target post-intervention probability of $P(Y|do(X), \vars Z, s^*)$ as accurately as possible. Notice that a subset of the observed variables $\vars O$ may be an s-admissible backdoor set, even if  $\vars O$ is not: For example, in Fig.\ref{fig:motivated_example}(c),
$\{Z, W\}$ is not an sABS, but $\{Z\}$ alone is.

One strategy would be to look through all possible subsets of $\vars O$, and return the one that maximizes the probability $P(\hz|D_e, D_o^*)$. However, doing this exhaustively can only scale up to a handful of variables. The following theorem shows that we only need to look into subsets of the Markov Boundary of the outcome in \graph G ($MB_{\graph G}(Y)$):

\begin{theorem}\label{the:proof_MB}
If there exists a set $\mathbf{Z}\subset \vars O$ that is sABS for $ X,Y$ with respect to \graph D,  there exists a set $\mathbf{Z^*}\subseteq MB_\graph G(Y)$ that is sABS for $ X,Y,$ with respect to \graph D.
\end{theorem}


This result limits the search space for possible s-admissible backdoor sets. Algorithm~\ref{algo:FindsABS} begins by identifying the Markov Boundary of $\var Y$ in the observational data (line~2), and then performs a greedy search 
to find the sABS $\vars Z^*$ that maximizes the marginal likelihood of the outcome in $D_e$. Starting from the empty set (line~4), the algorithm greedily adds or removes covariates from the Markov Boundary of $\var Y$ (lines~5--9). At each step, it selects the candidate set that yields the maximum increase in the marginal likelihood $P(D_e \mid D_o^*, h_{\vars Z})$ (line~10). In this way, the algorithm discards variables that either introduce discrepancies between the observational and experimental distributions, or become redundant for predicting $\var Y$ due to conditional independence. The process is repeated until no single-variable addition or removal improves the score, at which point the algorithm returns the set $\vars Z^*$ that achieves the maximum value (lines~11--14). Finally, if the probability that $\vars Z^*$ is an sABS exceeds a threshold $t$, Algorithm~\ref{algo:FindsABS} outputs the posterior expectation of $P(\var Y \mid do(X), \vars Z^*)$ using both $D_o^*$ and $D_e$ (lines~15--16); otherwise, it returns $\text{NaN}$.

We show the large-sample behavior of Alg. \ref{algo:FindsABS} for discrete data, where the scores $P(D_e|D_o^*, h_{\z})$ and $P(D_e|D_o^*, \neg  h_{\z})$ can be computed in closed form using the BD score \citep{heckerman1995learning}. Theorem \ref{the:proof_equality} shows that the marginal likelihood computed in Alg. \ref{algo:FindsABS} will asymptotically select an s-admissible backdoor set. Theorem \ref{the:proof_superset} shows that, asymptotically, adding a variable that is independent of the post-intervention outcome $Y|do(X), \vars s^*$ decreases the score. Thus, asymptotically the method avoids conditioning on irrelevant covariates, which improves efficiency by reducing variance while preserving unbiasedness of the estimator.

 
\begin{theorem}\label{the:proof_equality}
Assume Assumptions \ref{ass:main}, \ref{ass:sfaith} hold, $X, Y, \vars O$ are discrete, and $ N_o$ and $N_e$ increase equally without limit ($N:=N_e=N_o$ in the limit). Then Eq. \ref{eq:hzt} will converge to 1 if and only if \vars Z is an s-admissible backdoor set.
\begin{equation}
    \begin{cases}
    \displaystyle \lim_{N \rightarrow \infty}P(h_\vars{Z}|D_e, D_o^*) = 1, & \text{\vars Z is an sABS}   \\
    \smallskip
    \displaystyle \lim_{N \rightarrow \infty} P(h_\vars{Z}|D_e, D_o^*) = 0,& \text{ otherwise}
    \end{cases}
\end{equation}
\end{theorem}

\begin{theorem}\label{the:proof_superset}Assume Assumptions \ref{ass:main}, \ref{ass:sfaith}  hold, $X, Y, \vars O$ are discrete, and $ N_o$ and $N_e$ increase equally without limit ($N:=N_e=N_o$ in the limit). Let \vars Z, \vars Z' be s-admissible backdoor sets, $\vars Z\subset \vars Z'$, and \( (Y \perp\!\!\!\perp \vars Z'\setminus Z \mid \vars  Z)_{D_{\overline{X}}} \). Then, 
$$\lim_{N \rightarrow \infty}P(D_e|h_{\vars Z}, D_o^*)> \lim_{N \rightarrow \infty}P(D_e|h_{\vars Z'}, D_o^*)$$
\end{theorem}
Proofs of Theorems \ref{the:proof_MB},\ref{the:proof_equality}, \ref{the:proof_superset} are in the Supplementary.


\section{Related Work}\label{sec:related_work}
To our knowledge, our method is the first to (i) compute probabilities that a set is sABS and (ii) treat transportability as a feature selection problem. Our work has connections to several areas, outlined below.

\textbf{Identifiability/Adjustment} Several works focus on identifying post-intervention probabilities from observational data in the same domain, using graph knowledge or just observational data. (e.g. \citet{perkovic2017complete, smucler2020efficient, shpitser2006a,jaber2019causal, entner13a}. These works assume knowing the causal graph, or learning a set of graphs consistent with observational data. \cite{pmlr-v206-triantafillou23a} use observational and experimental data to compute the probability that a  set is a backdoor set. However, they do not allow for different domains, and exhaustively look through all possible subsets of the observed covariates.\\
\textbf{Transportability} The area of transportability focuses on generalizing causal knowledge from one or more source domains to a target domain. This problem was formally introduced by \cite{pearl_transportability_2011}, who defined selection diagrams and provided graphical conditions for transportability, such as s-admissibility, when the selection diagram is known. \cite{bareinboim_transportability_2012} provide a complete algorithm for computing  transport formulae, and  \citep{bareinboim_meta-transportability_2013} show that do-calculus is complete for transportability. These works form the theoretical framework for transferring causal knowledge across domains, \emph{when the selection diagram is known.}\\
\textbf{Combining data for effect estimation.} There is also growing body of work for combining observational and experimental data in the field of potential outcomes, focusing on using observational data to improve the RCT-based estimation. \citep{kallus2018removing, rosenman2020combining, cheng2021adaptive, wu22a, yang2023,  cheng2023double, parikh2023double}. However, these approaches rely on either unconfoundedness or s-admissibility (or both), and always return an estimator. In contrast, our method returns no estimator, if there is no reason to believe that an unbiased estimator exists. Moreover, all of these methods typically condition on the full set of covariates, ignoring the fact that s-admissibility/ignorability may hold only for a subset of the full set. In contrast, \texttt{FindsaBS} looks for the subset that leads to the most efficient unbiased prediction of the post-intervention outcome in the target domain.\\
\textbf{Statistical tests for unconfoundedness and transportability.}
Recent work has focused on testing the assumptions of conditional ignorability and s-admissibility using experimental and observational data. \cite{GaoYang2023pretest}, \cite{yang2023}, \cite{parikh2023double}, \cite{de-bartolomeis2024a}, develop frequentist tests for comparing estimates from RCTs and observational studies, but focus on average effects.  \cite{hussain_falsification_2023} proposes  a falsification test for conditional ignorability and transportability  based on conditional moment restrictions, and identifies covariate regions responsible for the violations. This approach has been extended to right-censored outcomes \citep{demirel_benchmarking_2024}. \cite{de-bartolomeis2024b} use observational and experimental data to identify subgroup-specific bias, allowing for a user-set  maximum bias within subgroups. If the conditional effects for some subgroup differ beyond this tolerance, the method rejects the null hypothesis, and determines that covariates are not an sABS. While these approaches typically operate on the full covariate set, they can be used to test whether any particular subset is an sABS. However, incorporating them into a search-based algorithm such as \texttt{FindsABS} is not straightforward, since p-values cannot be directly compared across different subsets (our method instead returns probabilities). We also point out that all methods mentioned here also implicitly assume sABS-faithfulness: if the observational and experimental estimators are not significantly different, they are assumed valid for the target population. In the experiments, we compare against \citet{de-bartolomeis2024b} in terms of Type I and Type II error. \citet{hussain_falsification_2023} were not evaluated, as no public implementation was available at the time of writing.\\
\textbf{Causal structure learning.} When the graph is unknown, one approach is to use causal discovery methods to find the causal structure, and then use graphical criteria to determine if a set is sABS. \citep{hyttinen2014constraint, andrews2020, triantafillou2015, mooij2019, hyttinen2015calculus} combine observational and experimental data, possibly from multiple domains, to learn the causal graph or answer queries for specific causal effects. However, we note that selection diagrams used in this work cannot be identified using data, as the selection variables cannot be distinguished. Moreover, these methods focus on answering if a $\vars Z$-specific causal effect is identifiable, but cannot select among different sets, while we select the set that maximizes the marginal likelihood of the experimental data. Developing constraint-based methods that also select an optimal transportable $\vars Z$-specific effect is interesting future work, and could work synergistically to our approach. In the supplementary, we compare against constraint-based causal discovery using the method in \cite{andrews2020}.

\section{Experimental Evaluation}\label{sec:experiments}

We evaluated Algorithms~\ref{algo:probsabs} and~\ref{algo:FindsABS} on both simulated and semi-synthetic data. Our evaluation covers two aspects: (i) identification of s-admissible backdoor sets, and (ii) causal effect estimation in the target domain. Code is available anonymously on \href{https://anonymous.4open.science/r/Transportability_without_Graphs-9FB3/}{GitHub}.
\vskip 1pt
\textbf{Identifying s-admissible backdoor sets}\vskip .2pt
\textbf{Simulated Data.}  
We first simulate data from the ground-truth graph in Fig.~\ref{fig:motivated_example}a, where each variable is modeled as a logistic regression of its parents. For each simulation, coefficients and intercepts are drawn uniformly from $[-2.5,-0.5]\cup[0.5,2.5]$ (binary variables and intercepts) or $[-1,-0.2]\cup[0.2,1]$ (continuous variables). We generate $N_o^* = 5000$ observational samples from the target distribution ($D_o^*$) and vary the number of experimental samples as $N_e \in \{50,100,300\}$. We repeat the process 100 times.

\textbf{\texttt{ProbsABS}}. We compute $P(h_{\mathbf{Z}}| D_e,D_o^*)$ for $\mathbf{Z}\in\{\emptyset,\{Z\},\{W\},\{Z,W\}\}$, and compute AUCs. 
In the ground-truth graph, only $\{Z,W\}$ is an sABS. Figure~\ref{fig:AUC_TPR_FPR_Compare_random} shows that \texttt{ProbsABS} successfully recovers this set across sample sizes.

\textbf{\citet{de-bartolomeis2024b}}.  
We compare against the public implementation of \citet{de-bartolomeis2024b}, which we refer to as \texttt{Bias-test}. This method tests whether experimental CATEs fall within $\tau_{os}\pm\delta$, for a user-specified tolerance $\delta$, where $\tau_{os}$ denotes the CATE estimated from the observational study. We follow the authors and use a constant $\delta$ for all subgroups, with $\delta\in\{0.01,0.03,0.05,0.1\}$. We note that our method does not involve any user-defined tolerance. Moreover, the null hypotheses are not directly comparable: \texttt{Bias-test} compares conditional average treatment effects (CATEs), whereas our method compares conditional distributions. Nonetheless, smaller values of $\delta$ correspond more closely to our hypothesis $h_{\mathbf{Z}}$. Since the public implementation of \texttt{Bias-test} does not allow $\delta=0$, we include values close to zero. A feature set is classified as an sABS if the null is \emph{not} rejected. Unlike \texttt{ProbsABS}, \texttt{Bias-test} produces a binary decision so we only report true and false positive rates. Fig.~\ref{fig:AUC_TPR_FPR_Compare_random} shows that while \texttt{ProbsABS} achieves high accuracy, \texttt{Bias-test} suffers from frequent false positives.

\begin{figure*}[h!]
    \centering
    \begin{tabular}{ccc}
        \includegraphics[width =0.3\textwidth]{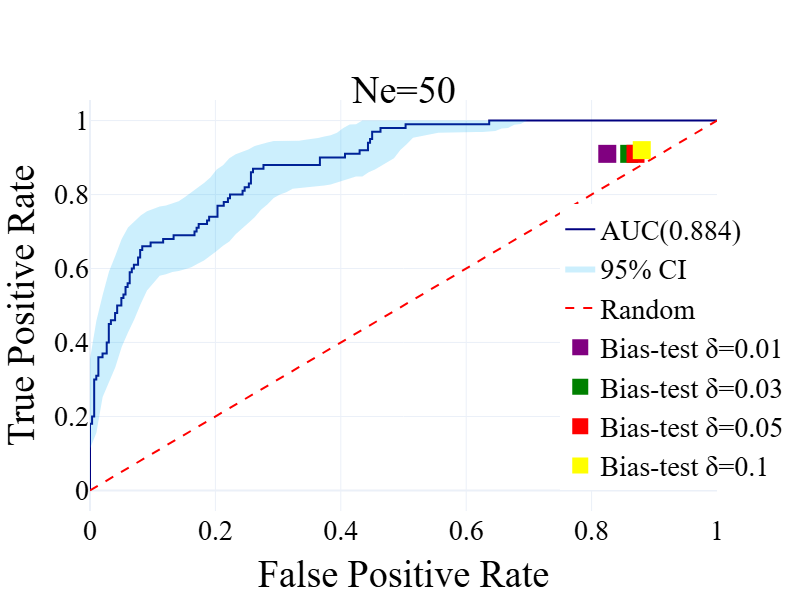}&
        \includegraphics[width =0.3\textwidth]{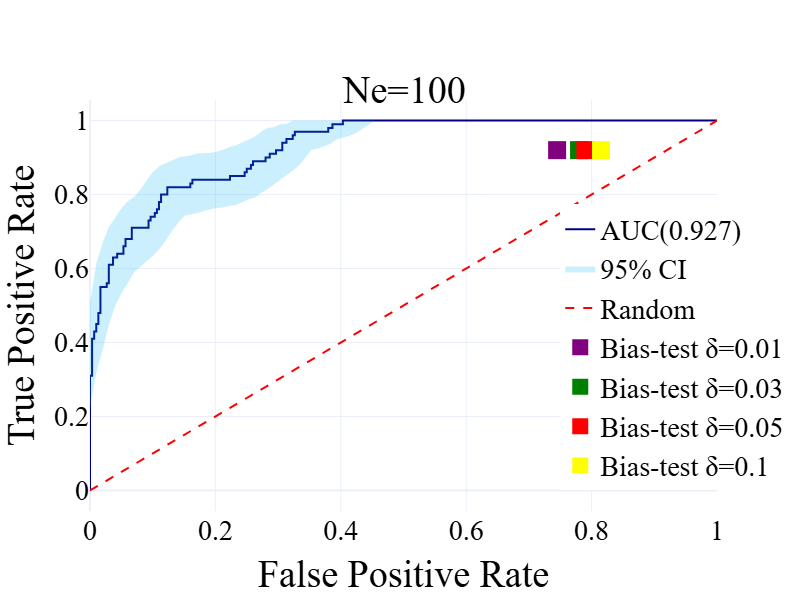}&
        \includegraphics[width =0.3\textwidth]{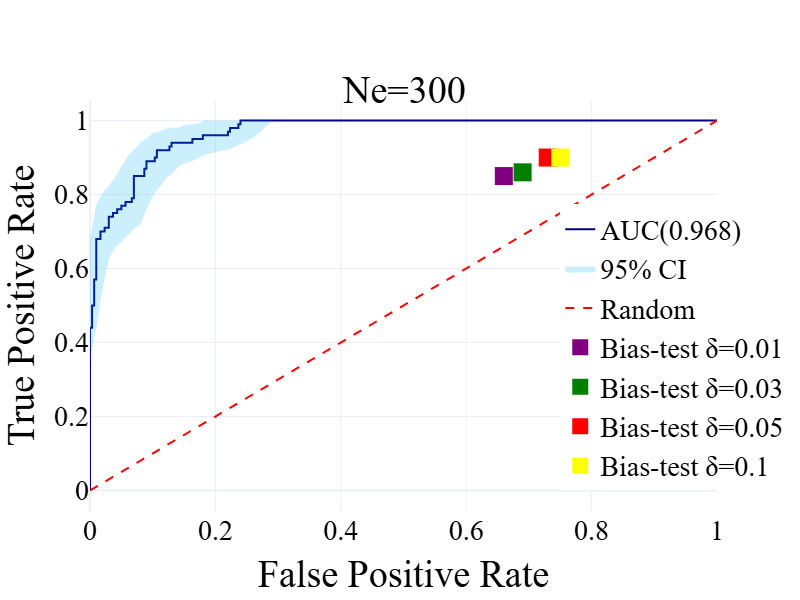}\\
    \end{tabular}
    \caption{\label{fig:AUC_TPR_FPR_Compare_random} \textbf{Areas under the ROC curve for classifying a set as sABS}. \texttt{ProbsABS} correctly classifies s-admissible backdoor sets, improving its performance as the experimental sample size increases. \texttt{Bias-test} suffers from high false positive rate.}
\end{figure*}

\textbf{Semi-synthetic Data (Hillstrom’s Email).}  
We next follow \citet{de-bartolomeis2024b} and construct semi-synthetic data based on the MineThatData Email dataset \citep{Hillstrom2008}, a randomized marketing experiment with $\sim$64,000 customers. The data include Treatment (receiving at least one email), Outcome (post-campaign spending),  and 13 covariates. We keep 20\% of the dataset as $D_e$ and use the remaining 80\% to generate a biased $D_o^*$ by (i) adding a constant shift of $30$ to treated units, and (ii) injecting bias $\delta^\star=60$ into selected subgroups. We consider:  
- \emph{Scenario 1 (single subgroup)}: a subgroup defined either by $\texttt{channel}=1 \,\&\, T=1$ (mode 1; 29.2\% of population) or by $\texttt{newbie}=1, \texttt{channel}=1, T=1$ (mode 2; 14.8\%).  
- \emph{Scenario 2 (multiple subgroups)}: 12 subgroups defined by \texttt{newbie}, \texttt{mens}, and \texttt{channel}, with varying biases up to $\delta^\star=60$.  

\begin{figure}[h!]
    \centering

    \begin{subfigure}{0.4\linewidth}
        \centering
        \caption{\textbf{Scenario 1, Mode 1}}
        \label{fig:scenario1_mode1}
        \resizebox{\linewidth}{!}{%
        \begin{tabular}{l c c c c c c}
        \toprule
         & \texttt{FindsABS} & \multicolumn{4}{c}{\texttt{Bias-test}} \\
        \cmidrule(lr){2-2} \cmidrule(lr){3-7}
        \textbf{$N_e$} & $t=0.5$ & $\delta=0.01$ & $\delta=2$ & $\delta=40$ & $\delta=58$ & $\delta=70$ \\
        \midrule
        50   & \textbf{20/20} & \textbf{20/20} & 14/20  & 3/20  & 1/20  & 1/20  \\
        100  & \textbf{20/20} & \textbf{20/20} & \textbf{20/20} & 12/20 & 8/20  & 4/20  \\
        300  & \textbf{20/20} & \textbf{20/20} & \textbf{20/20} & 17/20 & 14/20 & 11/20 \\
        1000 & \textbf{20/20} & \textbf{20/20} & \textbf{20/20} & \textbf{20/20} & 14/20 & 10/20 \\
        \bottomrule
        \end{tabular}%
        }
    \end{subfigure}
    \hspace{0.1\linewidth} 
    \begin{subfigure}{0.4\linewidth}
        \centering
        \caption{\textbf{Scenario 1, Mode 2}}
        \label{fig:scenario1_mode2}
        \resizebox{\linewidth}{!}{%
        \begin{tabular}{l c c c c c c}
        \toprule
         & \texttt{FindsABS} & \multicolumn{4}{c}{\texttt{Bias-test}} \\
        \cmidrule(lr){2-2} \cmidrule(lr){3-7}
        \textbf{$N_e$} & $t=0.5$ & $\delta=0.01$ & $\delta=2$ & $\delta=40$ & $\delta=58$ & $\delta=70$ \\
        \midrule
        50   & 5/20 & \textbf{9/20} & 0/20  & 0/20  & 0/20  & 0/20  \\
        100  & 5/20 & 18/20 & \textbf{20/20}  & 3/20  & 0/20  & 0/20  \\
        300  & 8/20 & \textbf{20/20} & \textbf{20/20} & 12/20 & 4/20  & 2/20  \\
        1000 & 10/20 & \textbf{20/20} & \textbf{20/20} & \textbf{20/20} & \textbf{20/20} & 19/20 \\
        \bottomrule
        \end{tabular}%
        }
    \end{subfigure}

    \vspace{0.5em} 

    \begin{subfigure}{\linewidth}
        \centering
        \caption{\textbf{Scenario 2}}
        \label{fig:scenario2}
        \resizebox{0.4\linewidth}{!}{
        \begin{tabular}{l c c c c c c}
        \toprule
         & \texttt{FindsABS} & \multicolumn{4}{c}{\texttt{Bias-test}} \\
        \cmidrule(lr){2-2} \cmidrule(lr){3-7}
        \textbf{$N_e$} & $t=0.5$ & $\delta=0.01$ & $\delta=2$ & $\delta=40$ & $\delta=58$ & $\delta=70$ \\
        \midrule
        50   & \textbf{20/20} & 18/20 & 0/20 & 1/20  & 0/20  & 1/20  \\
        100  & \textbf{20/20}& 18/20  & 5/20 & 4/20  & 2/20  & 1/20  \\
        300  & \textbf{20/20} & \textbf{20/20} & \textbf{20/20} & 12/20 & 12/20 & 9/20  \\
        1000 & \textbf{20/20} & \textbf{20/20} & \textbf{20/20} & 16/20 & 12/20 & 12/20 \\
        \bottomrule
        \end{tabular}%
        }
    \end{subfigure}

    \caption{\textbf{Rejection rates (out of 20 runs) for the hypothesis that the full set is an sABS in the semi-synthetic data.} 
    The full set is not an sABS. \texttt{ProbsABS} correctly rejects the hypothesis when the affected subgroups are large, even for small experimental sample sizes. 
    \texttt{Bias-test} suffers in small sample sizes, but is better in identifying small biased subgroups.}
    \label{fig:all_scenarios_horizontal}
\end{figure}

In all cases, the outcome mechanism differs across domains, so no true sABS exists. We apply \texttt{ProbsABS} and \texttt{Bias-test} to the full covariate set. With the full samples of $D_o^*$ and $D_e$, both methods correctly reject sABS; Tables~\ref{fig:scenario1_mode1}–\ref{fig:scenario2} show results for smaller sample sizes. For large biased subgroups (Scenario 1, mode 1; Scenario 2), \texttt{ProbsABS} rejects more reliably at small $N_e$, while \texttt{Bias-test} catches up as $N_e$ increases. For small biased subgroups (Scenario 1, mode 2), both methods struggle at low $N_e$, but \texttt{Bias-test} performs better under stricter tolerances, consistent with its design for detecting small-group biases. Average time for a single run was 70 seconds for \texttt{FindsABS} and 182 seconds for \texttt{Bias-test}.

\vskip 1pt
\textbf{Causal effect estimation}\vskip .2pt
Algorithm~\ref{algo:FindsABS} selects the most likely sABS $\mathbf{Z}^*$ and returns an estimator of $P(Y|do(X),\mathbf{Z}^*, D_e,D_o^*)$ if $P(h_\vars Z^*|D_e, D_o^*)>t$ (with $t=0.5$), and \textsc{NaN} otherwise. To our knowledge, no existing method does this; existing approaches always return an estimator based on a pre-specified covariate set.

We compare against the following  baselines: (i) $D_e$-only: $P(Y|do(X),\mathbf{O},D_e)$ (transportability), (ii)$D_o^*$-only: $P(Y|do(X),\mathbf{O},D_o^*)$ (unconfoundedness), and (iii)$D_e{+}D_o^*$: $P(Y| do(X),\mathbf{O},D_e,D_o^*)$ (both). These represent the best-case estimators for methods that rely on the corresponding assumptions without performing feature selection.

\textbf{Simulated Data.}  
We simulate from the graphs in Fig.~\ref{fig:motivated_example}a and~\ref{fig:motivated_example}d. In the first case, the only sABS is $\{Z,W\}$; in the second, the sABS are $\{W\}$ and $\emptyset$. Performance is evaluated by predicting $Y$ on an independent target experimental test set ($N^*_{\text{test}}=1000$) and reporting cross-entropy.

\begin{figure}[th!]
\centering
\begin{subfigure}{0.4\columnwidth}
        \centering
        \includegraphics[width=\linewidth]{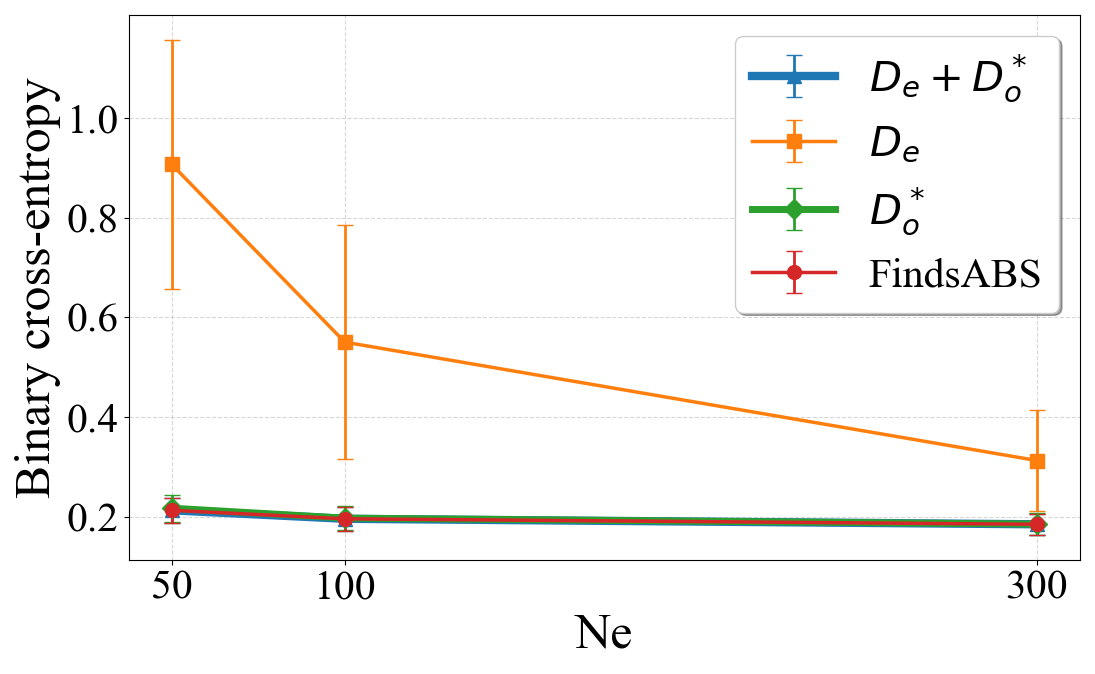}
        \caption{Data simulated from Fig. \ref{fig:motivated_example}a}
    \end{subfigure}
\hspace{0.1\linewidth} 
    \begin{subfigure}{0.4\columnwidth}
        \centering
        \includegraphics[width=\linewidth]{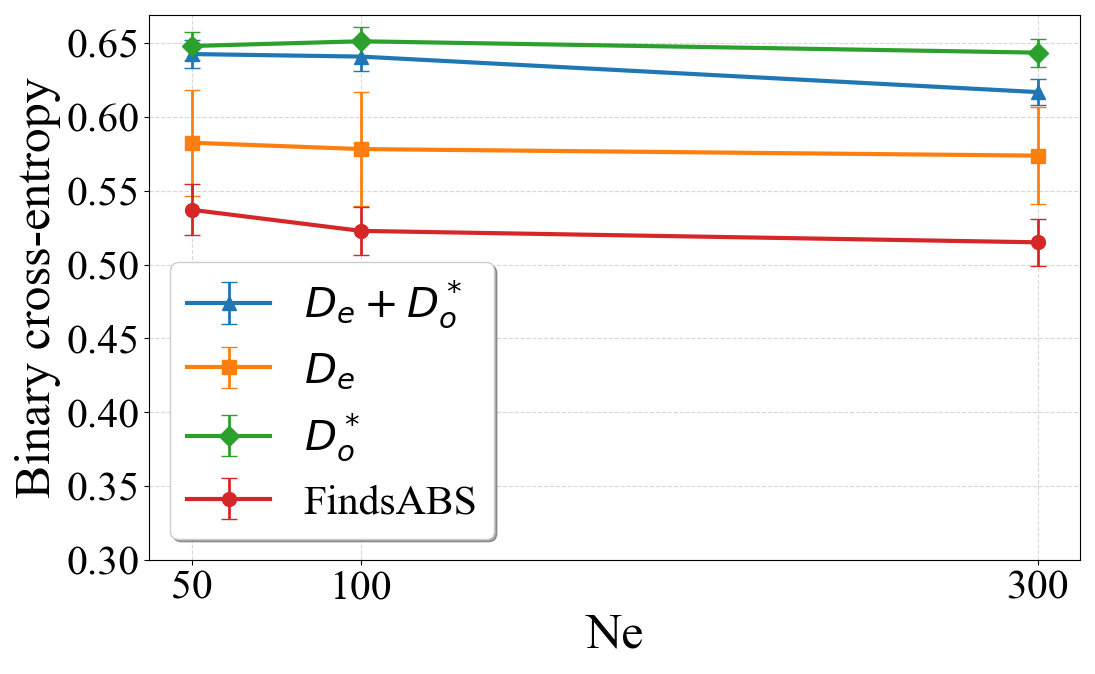}
        \caption{Data simulated from Fig. \ref{fig:motivated_example}d}
    \end{subfigure}
\caption{\textbf{ Binary cross-entropy for predicting the post-intervention outcome in the target domain.} (lower is better) (a)
$\{Z,  W\}$ is an sABS for $X, Y$, so all estimators based on observed covariates are unbiased. \texttt{FindsABS} performs on par with the  $D_o^*$ and $D_e+D_o^*$ -based estimators, and all outperform the  $D_e$ -based estimator, due to its smaller sample size. (b) Only $\{\emptyset\}$ and $\{W\}$ are sABS's.\texttt{FindsABS} identifies that $\{W\}$ is the best sABS and outperforms all estimators that condition on the full set.}
\label{fig:Case_1_performance}
\end{figure}

Results are shown in Fig.~\ref{fig:Case_1_performance}. For Fig.~\ref{fig:Case_1_performance}a, \texttt{FindsABS} matches the performance of $D_o^*$ and $D_e{+}D_o^*$ estimators, and all outperform $D_e$-only. For 
Fig.~\ref{fig:Case_1_performance}b, where $\{Z,W\}$ is not an sABS, \texttt{FindsABS} outperforms all baselines that condition on all variables. Additional experiments with all-discrete variables, larger covariate sets, and comparisons against \texttt{FCItiers} are provided in the Supplementary.

\section{Conclusions}\label{sec:conclusions}
We introduced a Bayesian, data-driven method to test whether conditional causal effects are transportable across domains without assuming a known causal graph. Our approach finds the most likely sABS to form a target-domain estimator. Across simulated and semi-synthetic studies, \texttt{FindsABS} (i) correctly identifies s-admissible backdoor sets and (ii) improves (or matches) target-based estimation when sABS exist, and returns \textsc{NaN} otherwise. To our knowledge, this is the first method to quantify sABS probabilities directly from data. Future work includes handling sample selection bias (common in clinical trials with strict inclusion/exclusion criteria) and applications to real-world data.

\bibliographystyle{plainnat}  
\bibliography{references}

\clearpage
\begin{center}
\textbf{\LARGE Supplementary Materials for:}\\[0.5em]
\textbf{\Large{Transportability without Graphs: A Bayesian Approach to Identifying s-Admissible Backdoor Set}}
\end{center}
\vspace{1em}

\setcounter{equation}{0}
\setcounter{figure}{0}
\setcounter{table}{0}
\renewcommand{\theequation}{S\arabic{equation}}

\section{Proof of Proposition \ref{prop:sabs}}
In this section, we provide the proof of Proposition \ref{prop:sabs} from the main paper. Throughout the document, unless otherwise mentioned, we assume Assumptions \ref{ass:main}.

\begin{repeatedprop}[\ref{prop:sabs}]
  Let \( D \) be the selection diagram characterizing \( \Pi \) and \( \Pi^* \), and let \(\vars S \) be the set of selection variables in \( D \). If $\vars Z \subseteq \vars V$ is an s-admissible backdoor set for $(X, Y)$ relative to \graph D, then \begin{equation}\label{eq:sabs}P(Y|do(X), \vars Z, \vars s) = P(Y|X,\vars Z, \vars s^*)
  \end{equation}
\end{repeatedprop}
\begin{proof}
\vars Z is s-admissible: 
\begin{equation}\label{eq:sadm}
    P(Y | do(X), \vars{Z}, \vars{s}) 
    = P(Y | do(X), \vars{Z}, \vars{s}^*) 
\end{equation}
\vars Z is a backdoor set in \graph D. $\graph G^*$ and \graph D only differ in the outgoing edges from \vars S to $\vars V \cup X\cup Y$. These edges cannot participate in backdoor paths in \graph D, so \vars Z blocks all backdoor paths in $\graph G^*$. Hence, $ P^*(Y | do(X), \vars{Z})  = P^*(Y | X, \vars{Z})$ so by definition:

\begin{equation}
    \label{eq:bdset}
P(Y | do(X), \vars{Z}, \vars{s}^*)  = P(Y | X, \vars{Z}, \vars{s}^*)
\end{equation}
\end{proof}

\section{Assumptions \ref{ass:sfaith} (sABS-faithfulness)}\label{sec:sabs_faitlfulness}

Assumptions \ref{ass:sfaith} states that (a) $P(X, Y, \vars O, \vars S)$ is faithful to \graph D, and (b)
\begin{equation}\label{eq:sfaith_add}
    \exists\;x, y, \vars z\textnormal{ s.t. } \frac{P(y|do(x), \vars z, \vars s)-P(y|do(x), \vars z, \vars s^*)}{ P(y|x, \vars z, \vars s^*)-P(y|do(x), \vars z, \vars s^*)}\neq 1
\end{equation}

Notice that based on faithfulness, if a set \vars Z is not s-admissible but is a backdoor set, then Eq. \ref{eq:bdset} holds, but 
\ref{eq:sadm} does not hold, Hence, Eq.\ref{eq:sabs} does not hold. Similarly, if \vars Z is an s-admissible but not a backdoor set, faithfulness implies Eq.\ref{eq:sabs} does not hold. However, if \vars Z is neither s-admissible nor a backdoor set, then it could be the case that Eq.\ref{eq:sabs} holds. This would happen if selection in the source domain has the exact same effect as confounding in the target domain, through some accidental parameter choices, hence, if Eq. \ref{eq:sfaith_add} holds.  The following corollary states that, under Ass. \ref{ass:sfaith}, Eq. \ref{eq:sabs} \emph{only holds} for sABS's.

\begin{corollary}
Under Assumptions \ref{ass:main}, \ref{ass:sfaith}, $P(Y|X, \vars Z, \vars s^*)=P(Y|do(X), \vars Z, \vars s)$ if and only if \vars Z is an s-admissible backdoor set for $(X, Y)$ with respect to \graph D. 
\end{corollary}

\begin{proof}
($\Rightarrow$) If \vars Z is an sABS, then  $P(Y|X, \vars Z, \vars s^*)=P(Y|do(X), \vars Z, \vars s)$ based on Proposition 
\ref{prop:sabs}.\\
($\Leftarrow$) If \vars Z is not an sABS, then \\
\textbf {Case 1:} If \vars Z is a backdoor set, but  not an s-admissible set: By faithfhulness if \vars Z is not an s-admissible set, $P(y|do(x), \vars z, \vars s^*) \neq P(y|do(x),\vars z, \vars s)$ for some $x, y, \vars z$. However since \vars Z is a backdoor set, $P(Y|do(X), \vars Z, \vars s^*) = P(Y|X,\vars Z, \vars s^*)$. Hence, $P(y|do(x), \vars z, \vars s)\neq P(y|x, \vars z, \vars s^*)$ for some $x, y, \vars z$ and Eq. \ref{eq:sabs} does not hold.

\textbf {Case 2: }Similarly, if \vars Z is not a backdoor set, but is an s-admissible set: $P(y|do(x), \vars z, \vars s^*) \neq P(y|x,\vars z, \vars s^*)$ for some $x, y, \vars z$, but $P(Y|do(X), \vars Z, \vars s^*) = P(Y|do(X),\vars Z, \vars s)$, so Eq. \ref{eq:sabs} does not hold.

\textbf {Case 3: }If \vars Z is not a backdoor set, and it is also not s-admissible, both inequalities hold, i.e.,
\[
\begin{aligned}
&\exists x, y, \vars z \quad P(y \mid do(x), \vars z, \vars s^*) \neq P(y \mid x, \vars z, \vars s^*), \\
&\exists x, y, \vars z \quad P(y \mid do(x), \vars z, \vars s^*) \neq P(y \mid do(x), \vars z, \vars s).
\end{aligned}
\]

Eq. \ref{eq:sabs} only holds if 
$$P(y|do(x), \vars z, \vars s^*)- P(y|do(x),\vars z, \vars s) = P(y|do(x), \vars z, \vars s^*) - P(y|x,\vars z, \vars s^*)\forall x, y, \vars Z,$$ hence, if $\nexists\;x, y, \vars z\textnormal{ s.t. } \frac{P(y|do(x), \vars z, \vars s)-P(y|do(x), \vars z, \vars s^*)}{ P(y|x, \vars z, \vars s^*)-P(y|do(x), \vars z, \vars s^*)}\neq 1$ which by Assumption\ref{ass:sfaith}(b) does not hold.

So, under Assumptions \ref{ass:sfaith}, if \vars Z is not an sABS, Eq. \ref{eq:sabs} does not hold.
\end{proof}

\section{Ablation study for the effect of $P(h_{\vars Z}|D_o^*)$.}\label{sec:abliationst}
One approach for computing the probability $P(h_{\vars Z}|D_o^*)$ is by reasoning on the space of possible causal graphs, which has been employed for testing conditional independencies \citep{claassen2012bayesian} and for ranking adjustment sets \citep{triantafillou_causal_2021}. However, in our case, the observational data do not carry enough information for the hypothesis \hz. Even if by using independence constraints or causal discovery methods we could uniquely identify that a set is a backdoor set in $\Pi^*$, observational data from a single domain do not carry enough information for the selection diagram, and in general cannot determine whether $\vars Z$ is s-admissible. 

\begin{figure}[ht!]
\centering
    \includegraphics[width =0.6\columnwidth]{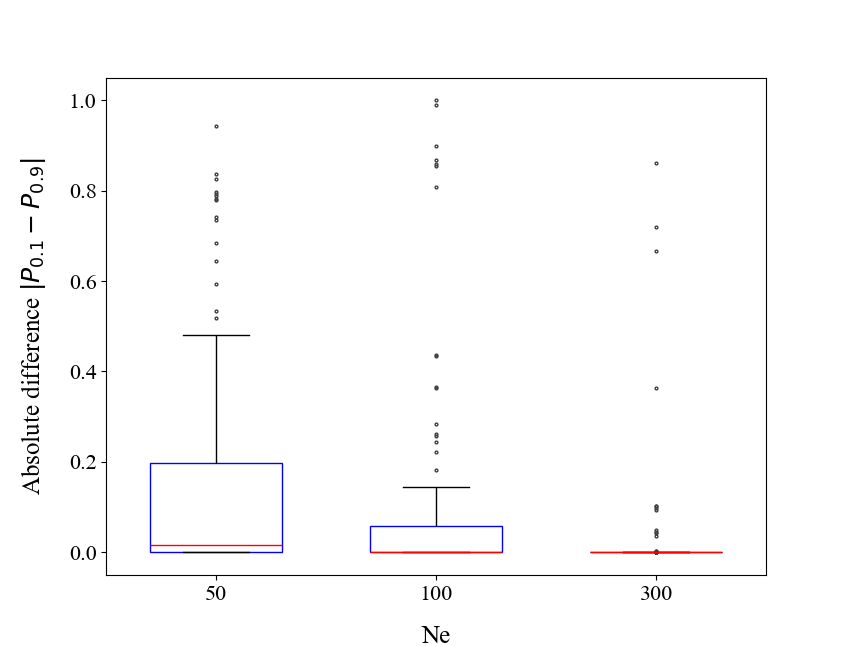}
    \caption{\label{fig:compare_priors}Effect of $P(h_{\z}|D_o^*)$ on Eq. \ref{eq:hzt}. $P_{0.1}= P(h_{\z}|D_e, D_o^*)$ computed using Eq. \ref{eq:hzt} with $P(h_{\z}|D_o^*)=0.1$, and $P_{0.9}= P(h_{\z}|D_e, D_o^*)$ computed using Eq. \ref{eq:hzt} with $P(h_{\z} |D_o^*)=0.9$. We used $N_o =5000$ and $N_e=50, 100, 300$ and we plot the distribution of $P_{0.1}-P_{0.9}$.}
\end{figure}

In practice, $P(h_{\z}|D_o^*)$ does not affect the behavior of the method, since its impact remains minimal even for small sample sizes. We illustrate this with an example: We use the structure in Fig.~\ref{fig:motivated_example}(a) and the settings from Section~\ref{sec:experiments} of the main paper, treating $\{Z\}$ as an observed confounder and assuming $\vars Z = \{Z, W\}$.  We then compute Eq. \ref{eq:hzt} with two different $P(h_{\z}|D_o^*)$: 0.1 and 0.9. We use $P_{0.1},P_{0.9} $ to denote Eq. \ref{eq:hzt} computed with $P(h_{\z}|D_o) = 0.1, 0.9$, respectively. Fig. \ref{fig:compare_priors} illustrates the distribution of the absolute difference $|P_{0.1}-P_{0.9}|$ over 100 random simulated parameters. The difference in  estimated  $P(h_{\z}|D_o^*, D_e)$ using  very different priors $P(h_{\z}|D_o^*)$ vanishes with increasing experimental sample size. Similar results are reported in \citet{triantafillou_causal_2021} and \citet{pmlr-v206-triantafillou23a}. For this reason, we use the uninformative $P(h_{\z}|D_o^*) = P(\neg h_{\z}|D_o^*)=0.5$ in this work.

\section{Formulae for computing Eq.\ref{eq:exp_score}}
In this section we present formulae for computing  Eq. \ref{eq:exp_score} in the main paper. Table \ref{tab:closed_form_disc} includes closed-form solutions for discrete data, and Table \ref{tab:equations_LR} shows the choice of priors and the formula for computing the likelihood for binary logistic regression models.
\begin{table}[h!]
 \caption{Closed-form solutions for Eq. \ref{eq:exp_score} in the main paper, for multinomial distributions with Dirichlet priors. Subscript $jk$ refers variable \var Y taking its $k$-th configuration, and variable set $\vars Z$ taking its $j$-th configuration. $\alpha_{jk}$ is the prior for the Dirichlet distribution. We set $\alpha_{jk}=1$ in all experiments. $N^o_{jk}, N^e_{jk}$ corresponds to counts in the data where $Y=k$ and  $\vars Z=j$ in $D_o^*$ and $D_e$, respectively. $N^o_{j}, N^e_{j}$ corresponds to counts in the data where $Z=j$. \label{tab:closed_form_disc}}
    \label{tab:equations}
\centering
    \begin{tabular}{c|c} 
    \toprule
    \\[-0.6em]
    Quantity &Analytical Expression\\[-0.6em] \\\hline
    \\[-0.6em]
    \(\displaystyle P(D_e|D_o,h_{\vars Z})\)   & \(\displaystyle \prod_{j=1}^q \frac{\Gamma(\alpha_j+N_j^o)}{\Gamma(\alpha_j+N_j^o+N_j^e)} \prod_{k=1}^r\frac{\Gamma(\alpha_{jk}+N_{jk}^o+N_{jk}^e)}{\Gamma(\alpha_{jk}+N_{jk}^o)}\)\\[-0.6em] \\ \hline
    \\[-0.6em]
      \(\displaystyle P(D_e|D_o,\neg h_{\vars Z})\)   & \(\displaystyle \prod_{j=1}^q \frac{\Gamma(\alpha_j)}{\Gamma(\alpha_j+N_j^e)} \prod_{k=1}^r\frac{\Gamma(\alpha_{jk}+N_{jk}^e)}{\Gamma(\alpha_{jk})}\) \\[-0.6em] \\
      \bottomrule
    \end{tabular}
\end{table}

\begin{table}[h!]
 \caption{Solutions based on logistic regression models for Eq.~\ref{eq:exp_score} in the main paper, considering a binary outcome $Y$, a binary treatment $X$, and mixed covariates $\vars Z=\{Z_1, \dots, Z_\textbf{k}\}$. We denote as $N_e$ the number of $D_e$. We denote as $(\theta_o^*)^j$ and $\theta_e^j$ the $j$-th sample of the observational and experimental parameters respectively. Cauchy distribution used as a weakly informative default prior distribution as proposed by \citet{gelman_weakly_2008}. The authors proposed this distribution because actual effects fall within a limited range. For instance, a typical change in an input variable would be unlikely to correspond to a change as large as 5 on the logistic scale (which would move the probability from 0.01 to 0.50 or from 0.50 to 0.99). For each sample $i$ in $D_e$, we have: $Y_i \sim \text{Bernoulli}(\pi_i)$ and we denote by $\pi_{ij}$  the probability of the i-th sample using the j-th parameter's sample. $\pi_i$ is approximated as the mean of the $\pi_{ij}$ values across all parameter samples. Assuming $N$ parameters' samples, we approximate the marginal likelihood of experimental data given the observational data as the average likelihood over all $N$ samples for $h_{\vars Z}$ and $\neg h_{\vars Z}$.}
\label{tab:equations_LR}
\centering
    \begin{tabular}{c|c} 
    \toprule
    \\[-0.6em]
    Quantity &Analytical Expression\\[-0.6em] \\\hline
    \\[-0.6em]
      \(\displaystyle (\theta_e)^j \)   & \(\displaystyle (\theta_{e0}^*)^j \sim Cauchy(0,10)\) , \(\displaystyle ((\theta_{e1}^*)^j,\dots, (\theta_{ek+1}^*)^j)\sim Cauchy(0, 2.5) \)\\[-0.6em] \\ \hline
    \\[-0.6em]
    
    \((\theta_o^*)^j \)   &  Cauchy priors for $\theta_o^*$, and then $(\theta_o^*)^j \sim f(\theta_o^*|D_o^*)$ using MCMC \\[-0.6em] \\ \hline
    \\[-0.6em]

    \(\pi_{ij} \)   & \(\frac{e^{\theta_{0}^j  + \theta_{1}^j Z_{i1} + ... + \theta_{k}^j  Z_{ik}+ \theta_{k+1}^jX_i}}{1 + e^{\theta_{0}^j  + \theta_{1}^j Z_{i1} + ... + \theta_{k}^j  Z_{ik}+ \theta_{k+1}^jX_i}}\)\\[-0.6em] \\ \hline
    \\[-0.6em]

    \(\displaystyle \graph P(D_e|(\theta)^j)\)   & \(\displaystyle \displaystyle\prod\limits_{i=1}^{N_e}  {\, \pi_{ij}^{Y_i} \, (1-\pi_{ij})^{(1-Y_i)}}\)\\[-0.6em] \\ 
      \bottomrule
    \end{tabular}
\end{table}

\section{Relaxing the assumption of \textbf{shared causal graphs}}\label{sec:Compare_no_shared_graph}

In the main paper, we define an s-admissible backdoor set (sABS) under the assumption that the populations $\Pi$ and $\Pi^*$ share the same causal graph, and Proposition~\ref{prop:sabs} is stated for these sABS sets. However, this proposition still holds when a set is s-admissible in a selection diagram $\graph D$ and a backdoor set for $(X, Y)$ in $\graph G^*$, but not in $\graph G$ (and consequently not in $\graph D$). This requires relaxing the shared-graph assumption in \cite{bareinboim_transportability_2012}’s definition of selection diagrams to permit structural differences across domains. Thus, letting \graph G and  $\graph G^*$ denote the causal graphs of the source and target domains respectively, we introduce a slightly different definition of selection diagrams and s-admissible backdoor sets:
\begin{definition}\label{ref:Definition2_relax}(Selection Diagram, relaxing the assumption of shared causal graphs)
Let \( \langle M, M^* \rangle \) be a pair of structural causal models relative to domains \( \langle \Pi, \Pi^* \rangle \), with corresponding causal diagrams \graph G and $\graph G^*$. \( \langle M, M^* \rangle \) is said to induce a selection diagram \graph D, if \graph D is constructed as follows:
\begin{enumerate}
    \item Every edge in \graph G  and every edge in $\graph G^*$ is also an edge in \graph D;
    \item D contains an extra edge $S_i \rightarrow V_i $ whenever there might exist a discrepancy $f_i \neq f_i^*$ or $P(V_i) \neq P^*(V_i)$ between $M$ and $M^*$.
\end{enumerate}
\end{definition}

Structural changes can now be represented through discrepancies between the functions $f_i$ and $f_i^*$. Definition \ref{ref:Definition2_relax} explicitly allows certain edges to be absent in either $\graph G^*$ or $\graph G$, as long as the resulting selection diagram $\graph D$ remains acyclic. After this relaxation, Proposition \ref{prop:sabs} from the main paper still holds as: (i) if \vars Z is s-admissible in \graph D: $ P(Y | do(X), \vars{Z}, \vars{s}) = P(Y | do(X), \vars{Z}, \vars{s}^*) $
and (ii) if \vars Z is a backdoor set in $\graph G^*$ : $ P(Y | do(X), \vars{Z}, \vars{s}^*)  = P(Y | X, \vars{Z}, \vars{s}^*)$ by definition:
\begin{equation*}
P(Y | do(X), \vars{Z}, \vars{s}^*)  = P(Y | X, \vars{Z}, \vars{s}^*)
\end{equation*}

By relaxing the assumption of shared causal graphs, two broader cases can be considered:

(a) when a common backdoor set \vars Z exists in both $P$ and $P^*$ populations despite the structural differences between them and is also s-admissible in \graph D (see Fig.~\ref{fig:no_shared_graph}, where $\{Z, W\}$ is a backdoor set in \graph G and $\graph G^*$ and s-admissible in \graph D.); and

(b) when a set $\vars Z$ is a backdoor set only in the target population (in $\graph G^*$), but not in the source (in $\graph G$ and hence not in the shared selection diagram $\graph D$) and is also s-admissible in \graph D (see Fig.~\ref{fig:no_shared_no_selection}, where $\{Z\}$ is a backdoor set in $\graph G^*$ but not in \graph G and is s-admissible in \graph D.).

However, allowing structural differences makes s-admissible backdoor sets non symmetric: if a set $\vars Z$ is s-admissible in $\graph D$ and backdoor in $\graph G^*$, then $\vars Z$ is a valid sABS for computing causal effects for domain $\Pi^*$ using our method. However, the same set may not be valid for computing causal effects in $\Pi$ if the data sources were reversed (i.e., we had observational data from   $\Pi$ and experimental from $\Pi^*$, and our target domain was $\Pi$). In this case, \vars Z may be a backdoor set in $\graph G^*$ but not in $\graph G$. This implies that we need to redefine sABS with respect to $\graph G, \graph G^*,\graph D$, and not just the selection diagram \graph D:

\begin{definition}{(s-Admissible Backdoor set relaxing the assumption of shared causal graphs, \textbf{sABS})}\label{def:s-admissible-backdoor-relax}\\
\noindent A set of variables $\vars Z$ is an \emph{s-admissible back-door set} for $(X,Y)$ relative to $\graph G$, $\graph G^*$ and $\graph D$ if
(i)  $(Y  \perp\!\!\!\perp  X \mid  \vars Z)_{G^*_{\underline{X}}}$ and
(ii) $(Y \perp\!\!\!\perp \vars S \mid \vars Z)_{D{_{\overline X}}}$.
\end{definition}

For example, in Fig.~\ref{fig:no_shared_no_selection}, $S_Z$ and $S_W$ denote distributional differences between populations in variables $\var Z$ and $\var W$, while $S_X$ indicates the absence of an edge from $\var W$ to $\var X$ in the target population. $\{Z, W\}$ forms a backdoor set in both $\graph G$ and $\graph G^*$. However, $\{Z, W\}$ is not an s-admissible set in $\graph D$, since conditioning on $\var W$ opens a collider path between $S_W$ and $\var Y$. Conversely, $\{Z\}$ is a backdoor set in $\graph G^*$ but not in $\graph G$ and $\graph D$. $\{Z\}$  is also s-admissible in \graph D. By relaxing the assumption of shared causal graphs and following Proposition~\ref{prop:sabs}, our method identifies ${Z}$ as an sABS, satisfying $P(Y \mid do(X), Z, \vars s) = P(Y \mid X, Z, \vars s^*)$. This implies that both $D_e$ from the source and $D_o^*$ from the target population can be used to obtain an unbiased estimate of $P(Y \mid do(X), Z, \vars s^*)$, even if $\{Z\}$ is not a backdoor set in \graph G.

\begin{figure*}[h!]
    \centering
    \begin{tabular}{ccc}
        \resizebox{0.15\textwidth}{!}{ 
            \begin{tikzpicture}[>=latex']
                \tikzstyle{every node}=[draw]
                \node (X) [circle, draw =none] at (0,0){$X$};
                \node (Z) [circle, draw =none, above = 0.7 of X]{$Z$};
                \node (mid)[draw=none, right =1 of X]{};
                \node (Y) [circle, draw =none, right = 0.7 of mid]{$Y$};
                \node (U) [circle, draw =none, above = 0.7 of Y]{$W$};
                \path (X) edge[->] (Y);
                \path (Z) edge[->] (Y);
                \path (Z) edge[->] (X);
                \path (U) edge[<->] (Y);
                \path (U) edge[->] (X);
            \end{tikzpicture}
        } &
        \resizebox{0.15\textwidth}{!}{ 
            \begin{tikzpicture}[>=latex']
                \tikzstyle{every node}=[draw]
                \node (X) [circle, draw =none] at (0,0){$X$};
                \node (mid)[draw=none, right =1 of X]{};
                \node (Z) [circle, draw =none, above = 0.7 of X]{$Z$};
                \node (Y) [circle, draw =none, right =  0.7 of mid]{$Y$};
                \node (U) [circle, draw =none, above = 0.7 of Y]{$W$};
                \path (X) edge[->] (Y);
                \path (Z) edge[->] (Y);
                \path (Z) edge[->] (X);
                \path (U) edge[<->] (Y);

            \end{tikzpicture}
        } &
        \resizebox{0.25\textwidth}{!}{ 
            \begin{tikzpicture}[>=latex']
                \tikzstyle{every node}=[draw]
                \node (X) [circle, draw =none] at (0,0){X};
                \node (mid)[draw=none, right =1 of X]{};
                \node (Z) [circle, draw =none, above =  .7 of X]{Z};
                \node (Y) [circle, draw =none, right =  1 of mid]{Y};
                \node (W) [circle, draw =none, above  = 0.7 of Y]{W};
                \node (Sz) [rectangle, draw=black, fill=black, text=white, above left =0.2 and 0.5 of Z] {}; 
                \node (S1) [rectangle, draw=none, left =0.1 of Sz] {$S_Z$};
                \node (Sx) [rectangle, draw=black, fill=black, text=white, left =0.5 of X] {}; 
                \node (S2) [rectangle, draw=none, left =0.1 of Sx] {$S_X$};
                \node (Sw) [rectangle, draw=black, fill=black, text=white, above right =0.2 and 0.5 of U] {}; 
                \node (S3) [rectangle, draw=none, right =0.1 of Sw] {$S_W$};
                \path (X) edge[->] (Y);
                \path (Z) edge[->] (Y);
                \path (Z) edge[->] (X);
                \path (W) edge[<->] (Y);
                \path (W) edge[->] (X);
                \path (Sz) edge[->] (Z);
                \path (Sx) edge[->] (X);
                \path (Sw) edge[->] (W);
            \end{tikzpicture}
        }  \\
        (a) \graph G in source domain $\Pi$ & (b) $\graph G^*$ in target domain $\Pi^*$  & (c) Selection diagram \graph D
    \end{tabular}
    \caption{\label{fig:no_shared_no_selection}\textbf{Different causal structures between domains.} $\{Z, W\}$ is a valid backdoor set in both \graph G and $\graph G^*$, but it is not an s-admissible set in \graph D. Thus, $\{Z, W\}$ is not an sABS. $\{Z\}$ is a valid backdoor set in $\graph G^*$ but not in \graph G and \graph D. $\{Z\}$ is an s-admissible set in \graph D. Relaxing the assumption of shared causal graphs, $\{Z\}$ is a valid sABS based on Def.~\ref{def:s-admissible-backdoor-relax} and both $D_e$, $D_o^*$ can be used to estimate the $Z$-specific causal effect in the target population.}
\end{figure*} 
Importantly, this asymmetry is directional: reversing the roles of the source and target populations would render $\{Z\}$ invalid as an sABS, even though the selection diagram $\graph D$ remains unchanged (Fig.~\ref{fig:no_shared_no_selection}(c)) and the $\vars Z$-specific causal effect is still transportable, $P(Y \mid do(X), Z, \vars s) = P(Y \mid do(X), Z, \vars s^*)$, because now $\{Z\}$ won't be a backdoor set in the new target population $\Pi^*$. Testing this transportability equation requires experimental data from both populations, which are not available. Consequently, in this reversed scenario, our method correctly identifies that $\{Z\}$ is not a valid sABS, since it is not a backdoor set in $\graph G^*$, and therefore $P(Y \mid do(X), Z, \vars s) = P(Y \mid X, Z, \vars s^*)$ does not hold.

\section{Experiments}
In this section, we give details on the Simulations of Section \ref{sec:experiments} and show additional experiments. In all simulations, we use a threshold of $t=0.5$ in Alg. \ref{algo:FindsABS}. Hence, we only return a causal effect if $P(h_{\vars Z}|D_e, D_o^*)>0.5$, and otherwise return NaN. In all the simulations we ran, $P(h_{\vars Z}|D_e, D_o^*)$ was at more than 0.5 for at least one set, so we always got an estimate.
\subsection{Simulation Details on Synthetic Data}

We first give some details for the simulations in Section \ref{sec:experiments}. We denote Simulations from Fig.~\ref{fig:motivated_example}(a) as \textbf{Case 1}, and Simulations from Fig.~\ref{fig:motivated_example}(d) as \textbf{Case 2}. For both cases, we assume a binary treatment and outcome, and mixed covariates. \textbf{Case 1:} We simulate $D_o^*, D_e$ from the selection diagram shown in Fig. \ref{fig:motivated_example}(a). 
$D_o^*$ for the target population simulated from the DAG, $\graph G$ ($\graph D$ without the selection variables), while $D_e$ for the source population simulated post-interventional DAG. $X, Y$ variables are binary and $Z, W$ continuous with $Normal(0,10)$ distributions. We use logistic regression to model the probability of the outcome given the treatment and covariates. The coefficients of the logistic regression model were randomly sampled in each iteration from the range \( [-2.5, -0.5] \cup [0.5, 2.5] \) for binary variables and intercept terms, and from \( [-1, -0.2] \cup [0.2, 1] \) for continuous variables. \textbf{Case 2:} We simulate $D_o^*, D_e$ from the selection diagram shown in  Fig. \ref{fig:motivated_example}(d). $X, Y, Z$ are binary variables and $H_1, H_2, W$ are continuous following $Normal(0, 1)$ distributions. For both $\Pi$ and $\Pi^*$, we set $Y, X$ to be functions of their parents using a parameter $\alpha=0.99$: $P(X=1|H_2>0.5) = \alpha$ and $P(Y = 1 \mid X = 1, H_1 > 0.5) = \alpha, \; \text{otherwise},$ \( Y \) follows a logistic model as a function of \( W \), with an intercept of \(-1\) and a slope of \(3\). Variable $Z$ has different distribution in $\Pi$ and $\Pi^*$ and defined as: $P(Z=1| H_1>0.5, H_2>0.5, \vars S= s)=1-\alpha$ and $P(Z=1| H_1>0.5, H_2>0.5, \vars S=s^*)=\alpha$. 

\subsection{Additional Experiments}\label{sec:additional_experiments}
We repeated the experiments in Sec~\ref{sec:experiments} using discrete data and the scores described in Table~\ref{tab:equations}. For Case 1, we simulated discrete data with random parameters. For Case 2, \var Z is a noisy OR of the two latent variables, and all other variables are discrete variables with random parameters.Binary cross entropy of the predictions of different methods are summarized in Fig.~\ref{fig:additional_exps_1}. AUCs for classifying s-admissible adjustment sets are shown in Fig. \ref{fig:aucs_discrete}.

\begin{figure}[th!]
\centering
\begin{tabular}{cc}
\includegraphics[width =0.4\columnwidth]{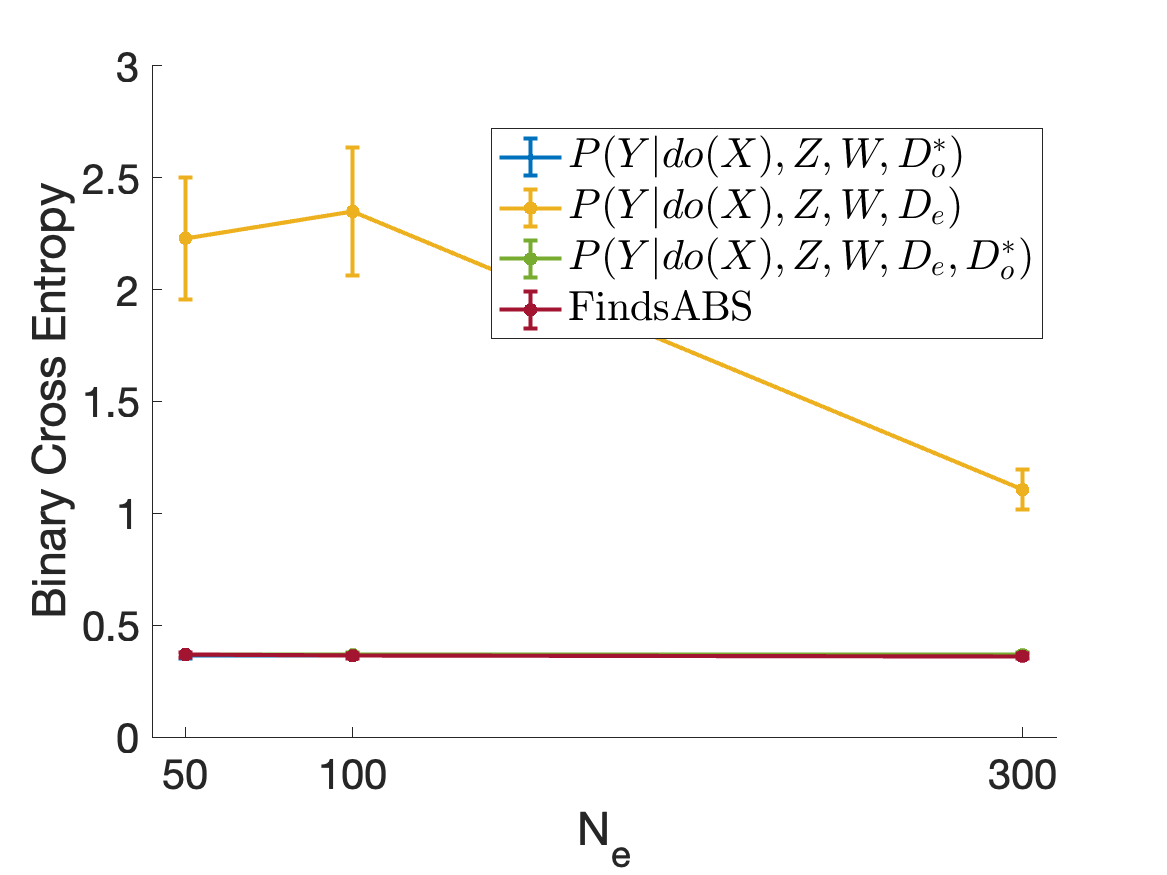}&
\includegraphics[width =0.4\columnwidth]{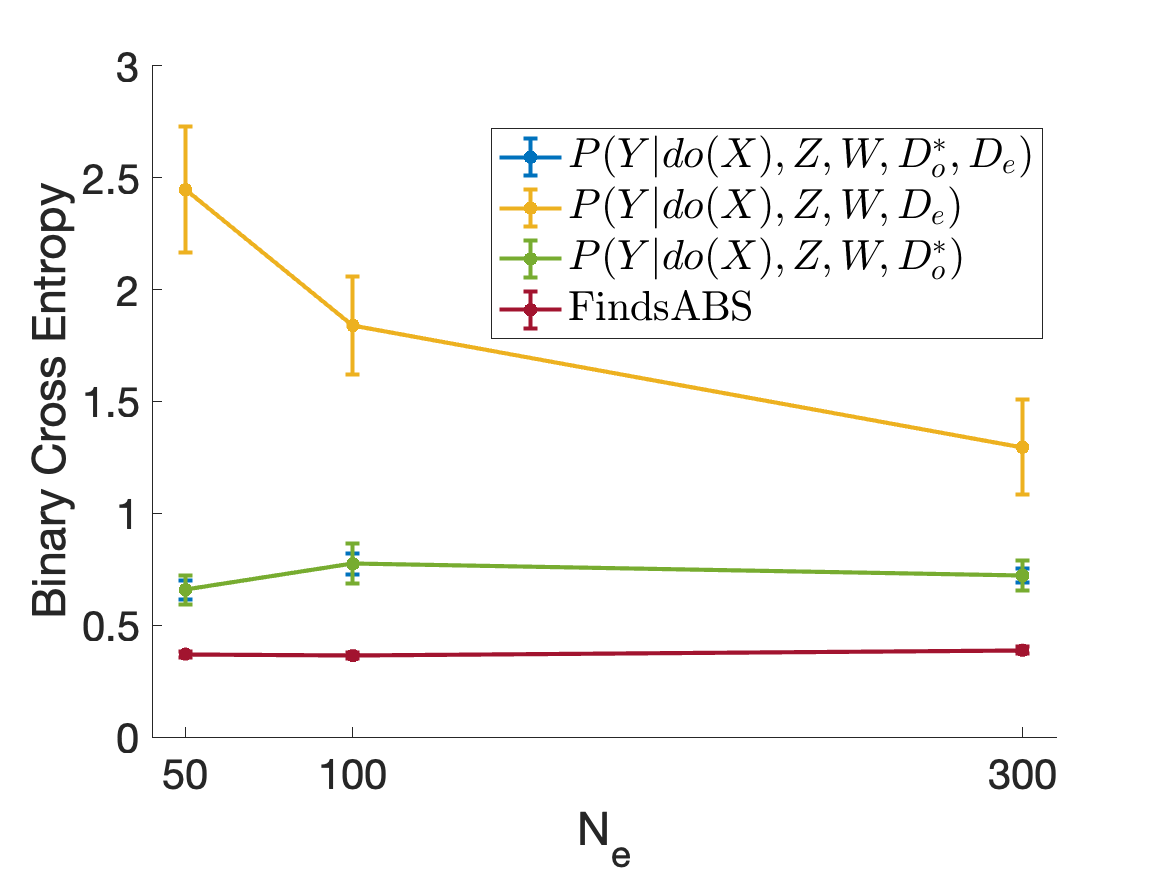}\\ (a) & (b)
\end{tabular}
\caption{Comparative performance of FindsABS and estimators based on $D_e$ and $D_o^*$. (a) Case 1: $\{Z, W\}$ is an sABS for $X, Y$, so all estimators are unbiased. FindsABS performs on par with the  $D_o^*$ and $D_e+D_o^*$ -based estimators, and all outperform the  $D_e$ based estimator. (b) Case 2: FindsABS outperforms $D_e$, $D_o^*$ and $D_e+D_o^*$ -based estimators using $(X,Z,W)$ as conditioning set. $D_o^*$ and $D_e+D_o^*$ -based estimators are almost identical.}
\label{fig:additional_exps_1}
\end{figure}

\begin{figure*}[h!]
    \begin{tabular}{ccc}
        $N_e=50$&$N_e=100$&$N_e=300$\\
        \includegraphics[width =0.3\columnwidth]{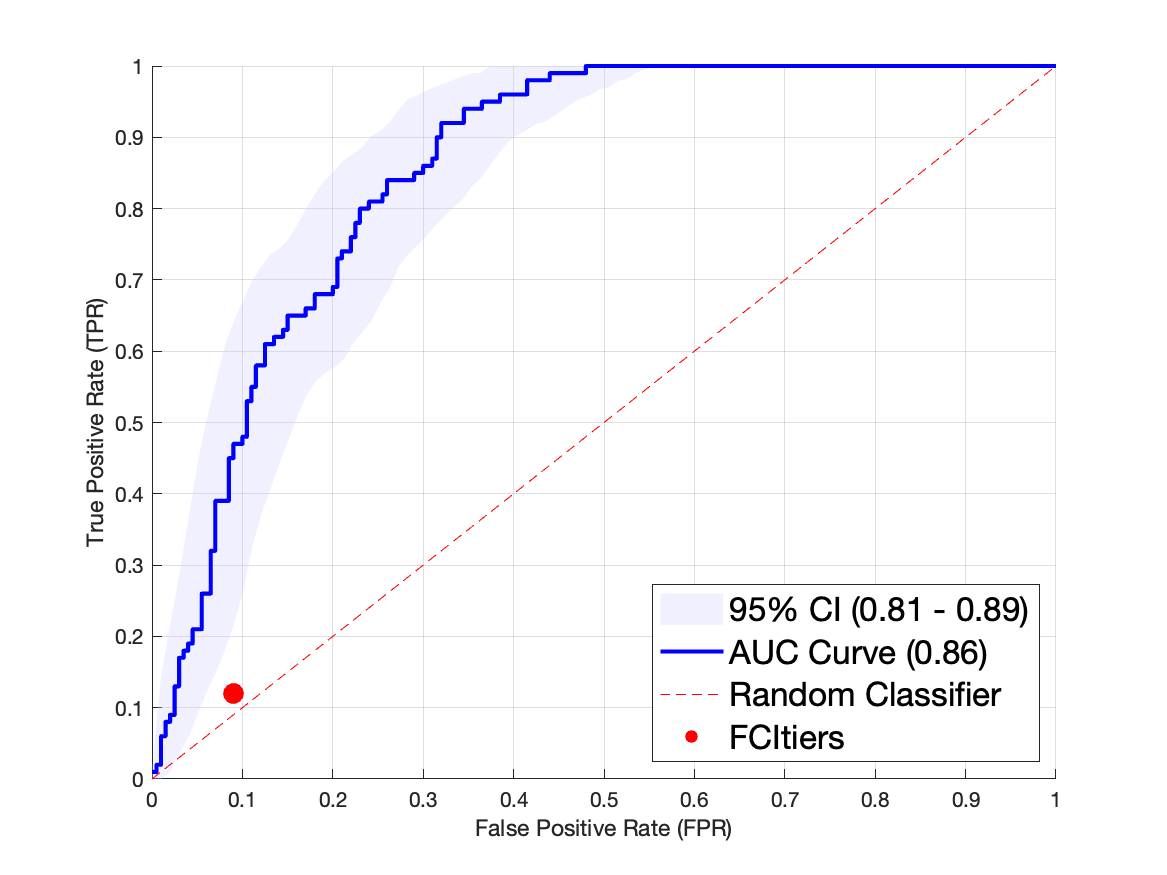}&
        \includegraphics[width =0.3\columnwidth]{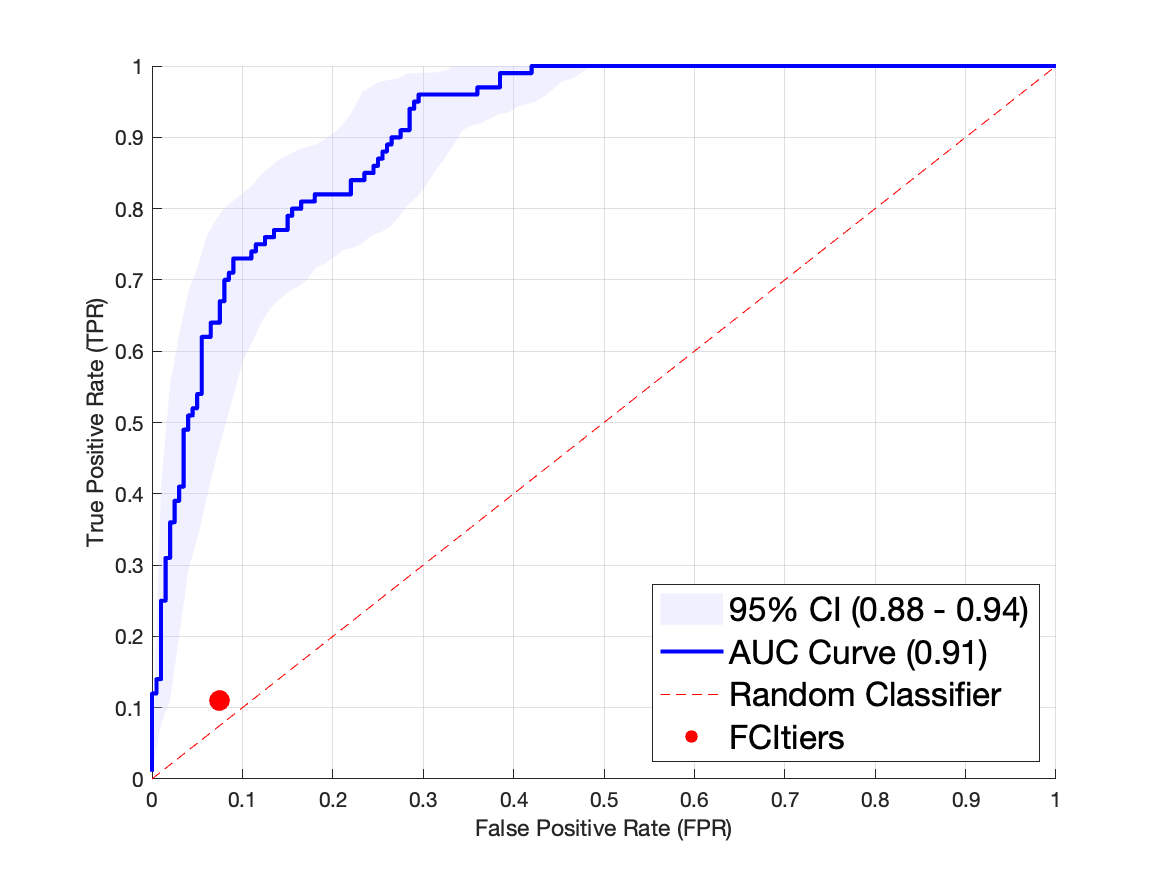}&
        \includegraphics[width =0.3\columnwidth]{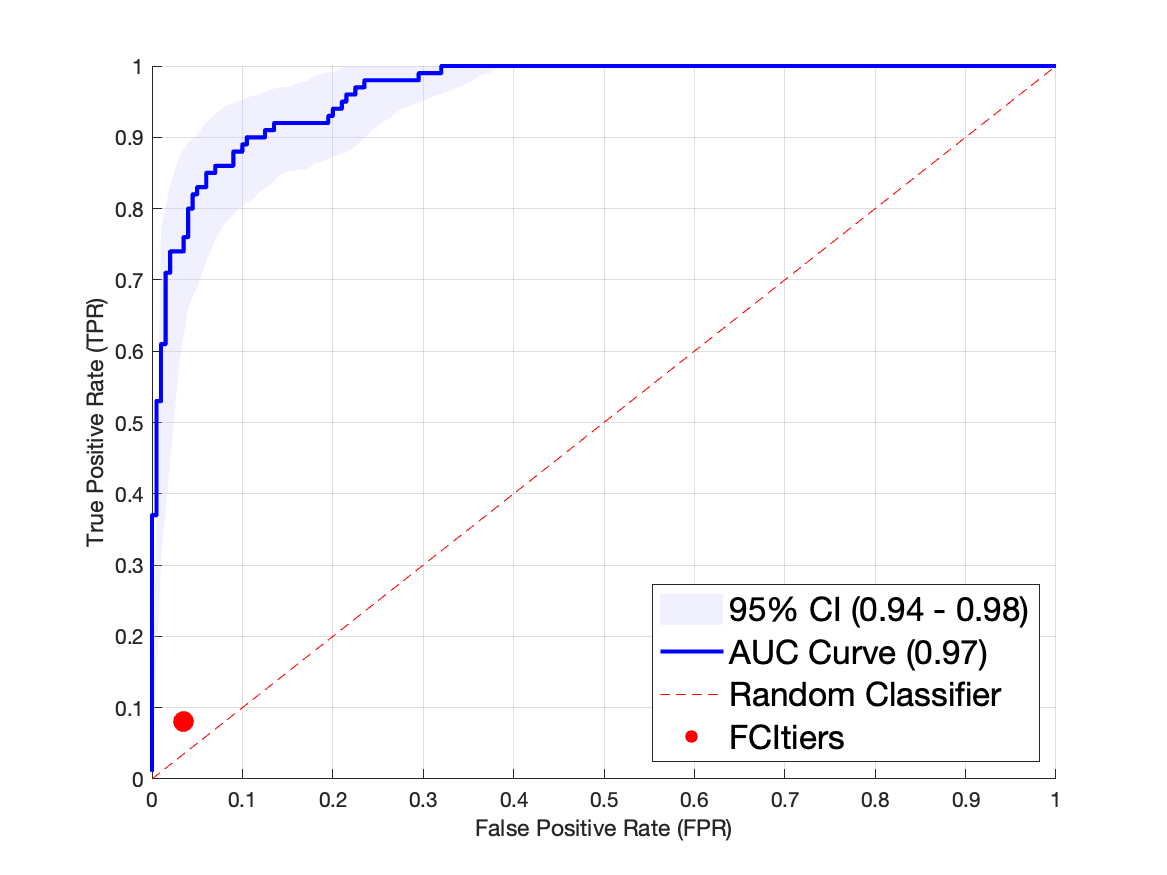}\\
    \end{tabular}
    \caption{Areas under the ROC curve for predicting $h_{ \textbf{Z}}$ with $5000$ in $D_o^*$ and an increasing number of samples in $D_e$, for discrete data. The red dot corresponds to the false positive and true positive rate of testing s-admissible backdoor sets using FCItiers. }
    \label{fig:aucs_discrete}
\end{figure*}

\subsection{Comparison with FCItiers}
As discussed in the related work section, there are several approaches that combine conditional independence relations in observational and experimental data to identify causal graphs. \citet{mooij2019} and \citet{andrews2020} discuss how to perform causal discovery from different domains and experiments. Specifically, one can model domain differences with an exogenous context variable $C$ and experiments with an exogenous instrumental variable \var I. Additionally, knowledge on the ordering of the variables can be included as knowledge tiers in FCI. \citet{andrews2020} shows that this version of FCI, called FCItiers, is complete in these settings.

We applied FCItiers with background knowledge in the data described in Sec. \ref{sec:additional_experiments}. We used a binary context variable $C$ to denote domain differences, and instrumental variable $I$ to model the randomization of $X$ in the source domain. \citet{andrews2020} suggest to add these variables in a tier that precedes all others. We also added a second tier that includes the pre-treatment variables, a third tier that includes the treatment, and a final tier including the outcome. Notice that $C$ and $I$ are identical in this setting, since randomization only happens in the source domain. In order to avoid any tests that include these two variables, we imposed that some edges are forbidden: Specifically, we did not allow any edges between $C$ and $I$, between $I$ and the pre-treatment covariates, between $I$ and the outcome, and between $C$ and $X$. We then ran FCItiers with this knowledge, to obtain the selection diagram \graph D. For each set $\vars Z$, we tested if \vars Z is an s-admissible backdoor set by testing (a) If $X$ and $Y$ are m-separated given $\vars Z$ in $D_{\underline X}$ (true for backdoor sets). and (b) If $C$ and $Y$ are m-separated given $\vars Z$ in $D_{\overline X}$ (true for s-admissible sets). If both conditions are satisfied, the set was predicted to be an s-admissible backdoor set. Since this method only outputs yes/no, we cannot compute AUCs. Instead, we computed the true positive rate and false positive rate for the predictions, and included it in the AUCs plots of Fig. \ref{fig:aucs_discrete}. As we can see, our method outperforms FCItiers in these settings.

\subsection{Scalability Evaluation: Experiments with More Variables}

To further evaluate the scalability and robustness of our method, we extended the experiments presented in Section \ref{sec:experiments} of the main paper to incorporate 10 variables. We simulate data from the ground-truth graph in Fig.~\ref{fig:exp_more_vars}, consisting of 6 continuous and 4 binary variables, where each variable is modeled as a logistic regression of its parents. Variables with identical distributions across populations ($Z_1$--$Z_7$) are represented as circles, with fixed distributions and regression coefficients summarized in Table~\ref{tab:covariates}. Variables \var Z and \var W differ between populations. In each iteration, the coefficients of the edges from \var Z and \var W to \var Y are resampled to test whether our algorithm can detect population differences when the effects of these variables vary. Coefficients are sampled from $[-1, -0.2] \cup [0.2, 1]$ for continuous variables and from $[-2.5, -0.5] \cup [0.5, 2.5]$ for binary variables and intercept terms. Continuous variables $Z$ and $W$ are drawn from normal distributions with their means sampled from \text{Uniform}(-5,5) and their standard deviations sampled from \text{Uniform}(0,5), ensuring differences in both mean and standard deviation between $P$ and $P^*$. $S_X$, encodes a mechanism change between the source and the target population.

\begin{figure*}[h!]
    \centering
    \begin{minipage}{0.5\textwidth} 
        \centering
        \scalebox{0.7}{ 
        \begin{tikzpicture}[>=latex']
            \tikzstyle{every node}=[draw]
            \node (X) [circle, draw =none] at (0,0){X};
            \node (mid)[draw=none, right =1 of X]{};
            \node (Z) [circle, draw =none, above =  .7 of X]{Z};
            \node (Y) [circle, draw =none, right =  1 of mid]{Y};
            \node (W) [circle, draw =none, above  = 0.7 of Y]{W};
            
            \node (Sz) [rectangle, draw=black, fill=black, text=white, above left =0.2 and 0.5 of Z] {}; 
            \node (S1) [rectangle, draw=none, left =0.1 of Sz] {$S_Z$};
            \node (Sx) [rectangle, draw=black, fill=black, text=white, left =0.5 of X] {}; 
            \node (S2) [rectangle, draw=none, left =0.1 of Sx] {$S_X$};
            \node (Sw) [rectangle, draw=black, fill=black, text=white, above right =0.2 and 0.5 of W] {}; 
            \node (S3) [rectangle, draw=none, right =0.1 of Sw] {$S_W$};

            \path (X) edge[->] (Y);
            \path (Z) edge[->] (Y);
            \path (Z) edge[->] (X);
            \path (W) edge[->] (Y);
            \path (W) edge[->] (X);
            \path (Sz) edge[->] (Z);
            \path (Sx) edge[->] (X);
            \path (Sw) edge[->] (W);

            \node (Z1) [circle, draw, right=1.2 of Y, below=0.5 of Y] {Z1};
            \node (Zdots) [draw=none, right=0.7 of Z1] {$\mathbf{\dots}$};
            \node (Z7) [circle, draw, right=0.7 of Zdots] {Z7};

            \path (Z1) edge[->] (Y);
            \path (Z7) edge[->] (Y);
        \end{tikzpicture}}
        \caption{Selection diagram including all 10 variables.}
        \label{fig:exp_more_vars}
    \end{minipage}%
    \hfill
    \begin{minipage}{0.4\textwidth} 
        \centering
        \small 
        \captionof{table}{Distributions and regression coefficients of additional covariates $Z_1$--$Z_7$.}
        \label{tab:covariates}
        \begin{tabular}{|c|c|c|}
            \hline
            Variable & Distribution & Coefficient $b_{Z_i \to Y}$ \\
            \hline
            $Z_1$ & Normal(3, 6) & 0.6 \\
            $Z_2$ & Normal(0, 3) & 0.9 \\
            $Z_3$ & Normal(-1, 4) & 0.3 \\
            $Z_4$ & Normal(-5, 5) & 0.2 \\
            $Z_5$ & Bernoulli(0.5) & 1.5 \\
            $Z_6$ & Bernoulli(0.8) & -1 \\
            $Z_7$ & Bernoulli(0.3) & -0.9 \\
            \hline
        \end{tabular}
    \end{minipage}
\end{figure*}

The binary cross-entropy of predictions obtained by \texttt{FindsABS}, $D_o^*$-, and $D_e$-based estimators is summarized in Fig.~\ref{fig:results_more_vars}. We consider two scenarios. In the first (Fig.~\ref{fig:results_more_vars}a), all 10 variables are observed; thus, every set that includes $\{Z, W\}$ is an sABS. \texttt{FindsABS} identifies the set that leads to the best post-intervention outcome and performs similarly to the $D_o^*$-based estimator. The $D_e$-based estimator performs similarly as $N_e$ increases. In the second,  Fig.\ref{fig:results_more_vars}(b), the covariates $Z$ and $W$—which differ in distribution between populations—are latent, so no sABS exists. Ideally, in this scenario, our algorithm should not produce any predictions. Nevertheless, Fig.\ref{fig:results_more_vars}(b) shows that even when our algorithm incorrectly identifies an sABS, its performance remains comparable to that of the best estimator. This may be because variables $Z_1$–$Z_7$ have a stronger effect on the outcome than $Z$ and $W$, or because the distributions of $Z$ and $W$ do not differ substantially between populations. When using $N_e=300$, our algorithm correctly identifies the absence of an SABs set in 2 out of 20 runs, and returns no predictions. For $N_e = 1000$ and  $N_e = 2000$, this occurs in 9 and 5 out of 20 runs, respectively. For the remaining cases where the algorithm returns predictions, the results are illustrated in Fig.\ref{fig:results_more_vars}(b).

\begin{figure}[h!]
\centering
\begin{tabular}{cc}
\includegraphics[width =0.4\columnwidth]{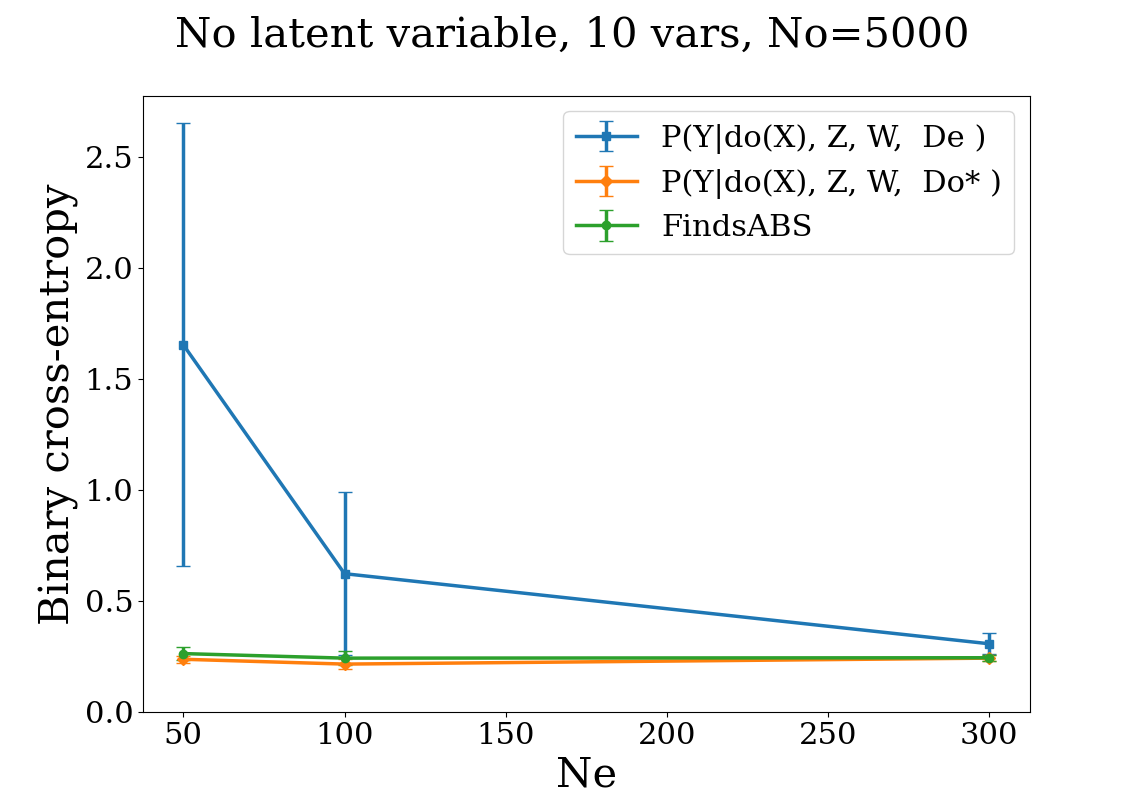}&
\includegraphics[width =0.4\columnwidth]{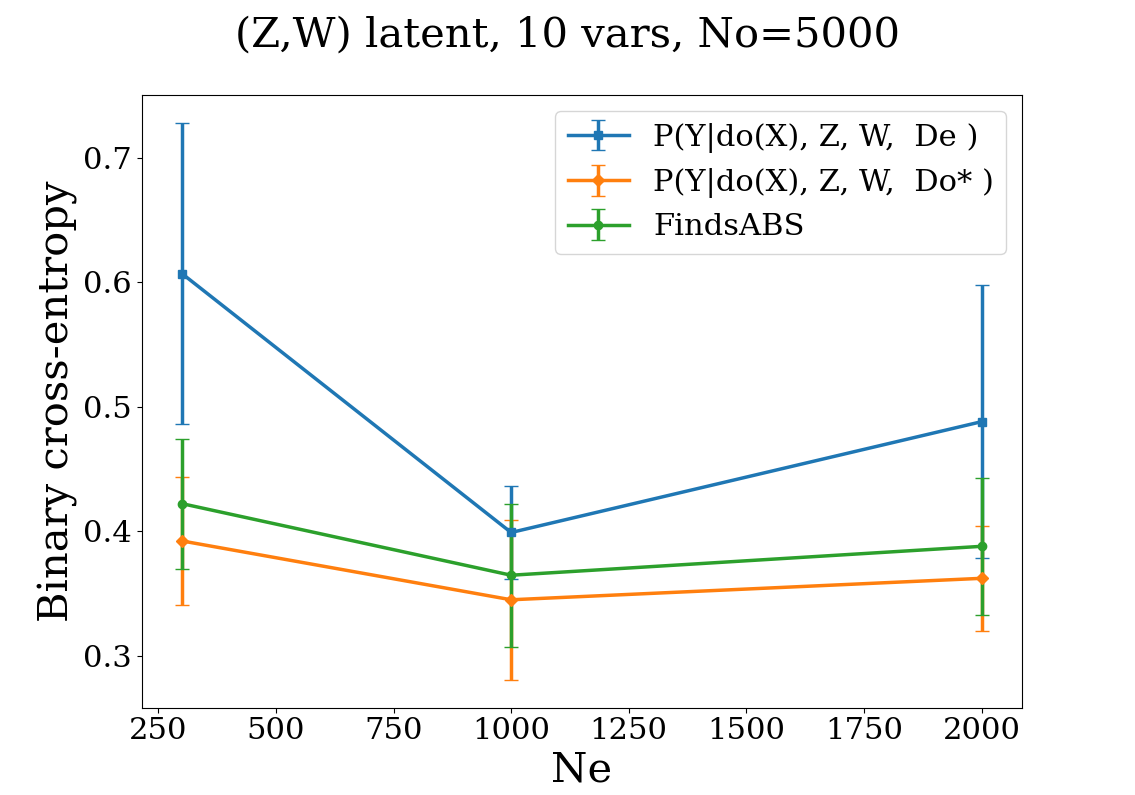}\\ (a) & (b)
\end{tabular}
\caption{Comparative performance of FindsABS and estimators based on $D_e$ and $D_o^*$. (a) Scenario 1: $\{Z, W, Z_1,...,Z_7\}$ is an sABS for $X, Y$, so all estimators are unbiased. FindsABS performs on par with the  $D_o^*$-based estimator, and outperforms the $D_e$-based estimator. (b) Scenario 2:  $\{Z, W\}$ are latent, thus $\{Z_1,..., Z_7\}$ is not an sABS. When FindsABS incorrectly identifies an sABS set, its performance remains comparable to the best estimator.}
\label{fig:results_more_vars}
\end{figure}

\subsection{Comparison with the implementation from \emph{\cite{de-bartolomeis2024b}}}\label{subsec:SubGroup_bias}


We compare our method with the implementation from \cite{de-bartolomeis2024b}, referred to as \texttt{Bias-test}, under different scenarios. In the main paper, experiments used $N_o^* = 5000$ and varying $N_e$. Here, in both Subsections~\ref{subsec:Compare1} and \ref{subsec:Compare2}, we increase $N_o^*$ and vary $N_e$ to examine methods performance with larger sample sizes for both simulated and semi-synthetic data.

\textbf{Overview of the \texttt{Bias-test} Method}\\
Before presenting the additional experimental results, we briefly outline the key components of the method proposed by \cite{de-bartolomeis2024b}:

For a distribution index $\diamond \in \{\mathrm{rct}, \mathrm{os}\}$ with underlying data-generating distribution $P_\diamond$, \cite{de-bartolomeis2024b} define the (conditional) treatment-effect regression function as:
\[
\tau_\diamond : \mathcal{X} \to \mathbb{R}, 
\qquad 
\tau_\diamond(x) := \mathbb{E}_{P_\diamond}[\,Y \mid T=1, X=x\,] \;-\; \mathbb{E}_{P_\diamond}[\,Y \mid T=0, X=x\,].
\]

The \textbf{Null hypothesis} defined as:
Let $\mathcal{X}$ be the feature space, $X \in \mathcal{X}$, $T \in \{0,1\}$ the treatment, and $Y \in \mathbb{R}$ the outcome and let patient subgroups defined via a subset of features $X_J$, corresponding to the covariates with indices $J \subseteq \{1,2,\dots,d\}$.The null hypothesis is given by:
\[
H_0: \; \mathbb{E}_{P_{\text{rct}}} \big[\tau_{\text{rct}}(X) \mid X_J \big] \in 
\Big(
\mathbb{E}_{P_{\text{rct}}} \big[\tau_{\text{os}}^-(X) \mid X_J \big], \;
\mathbb{E}_{P_{\text{rct}}} \big[\tau_{\text{os}}^+(X) \mid X_J \big]
\Big), \quad P_{\text{rct}} \text{-a.s.}
\]

where $\tau^{\pm}_{\mathrm{os}}: \mathcal{X} \to \mathbb{R}$ are the two bounded tolerance functions that capture how much the estimated treatment effects can differ between studies, and which satisfy $\tau^{-}_{\mathrm{os}}(x) \;\leq\; \tau_{\mathrm{os}}(x)\leq\; \tau^{+}_{\mathrm{os}}(x),
 \; \forall\, x \in \mathcal{X}.$

In \texttt{Bias-test} method, the user specifies a tolerance function $\delta(x)$, which defines the maximum permissible difference between the RCT-estimated treatment effect and the observational estimate (after correcting for bias), before rejecting the null hypothesis: $
\tau_{\mathrm{os}}^{\pm}(x) = \tau_{\mathrm{os}}(x) \pm \delta(x)$. In the experiments below, we follow their approach by selecting a constant $\delta$ ($user\_shift$). 

For the purposes of comparison, \textbf{we classify any feature set for which the null hypothesis is rejected as \textit{non-sABS}}, while all other sets are classified as \textit{sABS}. In this framework, a true positive occurs when the test correctly fails to reject $H_0$ for a set that is in fact an sABS, whereas a false negative occurs when such a set is incorrectly classified as non-sABS. Conversely, when the set is not an sABS, rejecting $H_0$ yields a true negative, while failing to reject it results in a false positive.

\subsubsection{Additional Comparison Using Simulated Data with Relaxed Shared-Graph Assumption}\label{subsec:Compare1}

We simulate data according to the scenario in Fig.~\ref{fig:no_shared_graph}(c). In the source population $\mathcal{P}$, covariates are generated as $Z \sim \mathcal{N}(1, 3)$ and $W \sim \text{Bernoulli}(0.9)$, while in the target population $\mathcal{P}^*$, $Z \sim \mathcal{N}(0, 5)$ and $W \sim \text{Bernoulli}(0.3)$. Treatment in the observational dataset from the target population, $D_o^*$, is generated via $\text{logit}(X) = -1 + 0.8 Z + 2.1 W$, whereas in the experimental dataset from the source, $D_e$, treatment is randomized. Outcomes are generated in both populations using 
$\text{logit}(Y) = 0.9 Z + 1.9 W + 2.2 X + e$, $e \sim N(0,1)$. 
We generate 20 datasets with fixed seeds and we evaluate whether the two methods can correctly identify if a set is an sABS.

\begin{figure*}[h!]
    \centering
    \begin{tabular}{ccc}
        \resizebox{0.15\textwidth}{!}{ 
            \begin{tikzpicture}[>=latex']
                \tikzstyle{every node}=[draw]
                \node (X) [circle, draw =none] at (0,0){$X$};
                \node (Z) [circle, draw =none, above = 0.7 of X]{$Z$};
                \node (mid)[draw=none, right =1 of X]{};
                \node (Y) [circle, draw =none, right = 0.7 of mid]{$Y$};
                \node (U) [circle, draw =none, above = 0.7 of Y]{$W$};
                \path (X) edge[->] (Y);
                \path (Z) edge[->] (Y);
                \path (Z) edge[->] (X);
                \path (U) edge[->] (Y);
            \end{tikzpicture}
        } &
        \resizebox{0.15\textwidth}{!}{ 
            \begin{tikzpicture}[>=latex']
                \tikzstyle{every node}=[draw]
                \node (X) [circle, draw =none] at (0,0){$X$};
                \node (mid)[draw=none, right =1 of X]{};
                \node (Z) [circle, draw =none, above = 0.7 of X]{$Z$};
                \node (Y) [circle, draw =none, right =  0.7 of mid]{$Y$};
                \node (U) [circle, draw =none, above = 0.7 of Y]{$W$};
                \path (X) edge[->] (Y);
                \path (Z) edge[->] (Y);
                \path (Z) edge[->] (X);
                \path (U) edge[->] (Y);
                \path (U) edge[->] (X);

            \end{tikzpicture}
        } &
        \resizebox{0.25\textwidth}{!}{ 
            \begin{tikzpicture}[>=latex']
                \tikzstyle{every node}=[draw]
                \node (X) [circle, draw =none] at (0,0){X};
                \node (mid)[draw=none, right =1 of X]{};
                \node (Z) [circle, draw =none, above =  .7 of X]{Z};
                \node (Y) [circle, draw =none, right =  1 of mid]{Y};
                \node (W) [circle, draw =none, above  = 0.7 of Y]{W};
                \node (Sz) [rectangle, draw=black, fill=black, text=white, above left =0.2 and 0.5 of Z] {}; 
                \node (S1) [rectangle, draw=none, left =0.1 of Sz] {$S_Z$};
                \node (Sx) [rectangle, draw=black, fill=black, text=white, left =0.5 of X] {}; 
                \node (S2) [rectangle, draw=none, left =0.1 of Sx] {$S_X$};
                \node (Sw) [rectangle, draw=black, fill=black, text=white, above right =0.2 and 0.5 of U] {}; 
                \node (S3) [rectangle, draw=none, right =0.1 of Sw] {$S_W$};
                \path (X) edge[->] (Y);
                \path (Z) edge[->] (Y);
                \path (Z) edge[->] (X);
                \path (W) edge[->] (Y);
                \path (W) edge[->] (X);
                \path (Sz) edge[->] (Z);
                \path (Sx) edge[->] (X);
                \path (Sw) edge[->] (W);
            \end{tikzpicture}
        }  \\
        (a) \graph G in source domain $\Pi$ & (b) $\graph G^*$ in target domain $\Pi^*$  & (c) Selection diagram \graph D
    \end{tabular}
    \caption{\label{fig:no_shared_graph} $S_Z$ and $S_W$ variables model differences in the distributions of $Z$ and $W$ between $\Pi$ and $\Pi^*$, while $S_X$ encodes a mechanism change between the source and the target population. $\{Z, W\}$ is an s-admissible set and a backdoor set in $\graph D$. Thus, $P(Y \mid do(X), Z, W, \mathbf{S} = \mathbf{s})$ is transportable from $\Pi$ to $\Pi^*$.}
\end{figure*}

Fig.~\ref{fig:AUC_TPR_FPR_Compare} shows the AUCs for classifying s-admissible backdoor sets for $N_o^*=5000$ and $N_e \in \{50, 100, 300\}$.
For FindsABS, AUCs are computed across all subsets of covariates, including the empty set. Since \texttt{Bias-test} method produces only binary (yes/no) outputs, AUCs cannot be computed. Instead, we calculate the true positive rate (TPR) and false positive rate (FPR) for predictions using different $\delta \in \{ 0.03, 0.05, 0.1\}$ parameters and include these in the AUC plots. As the \texttt{Bias-test} method cannot be applied to the empty set, TPR and FPR are computed only for the non-empty subsets $\{\{Z\}, \{W\}, \{Z, W\}\}$. Fig.~\ref{fig:AUC_TPR_FPR_Compare} shows that while \texttt{ProbsABS} achieves high accuracy, \texttt{Bias-test} suffers from frequent false positives.

\begin{figure*}[h!]
    \begin{tabular}{ccc}
        \includegraphics[width =0.3\columnwidth]{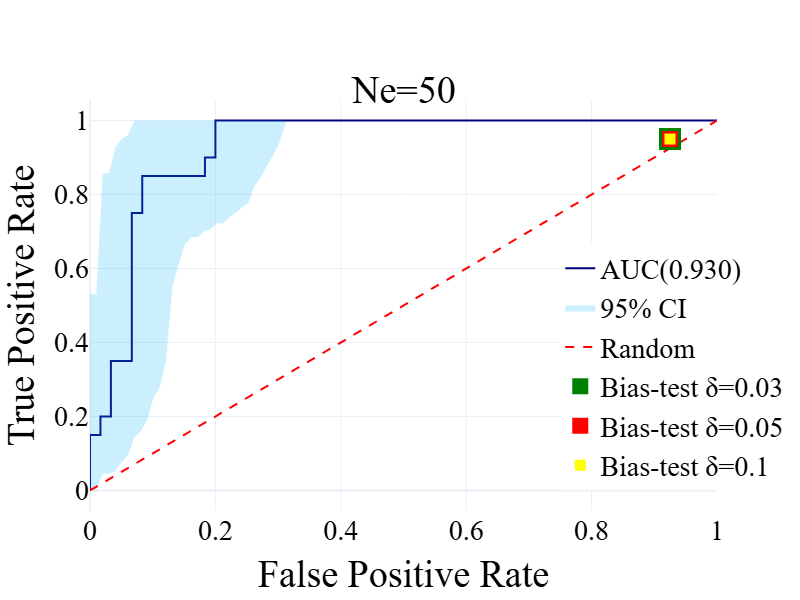}&
        \includegraphics[width =0.3\columnwidth]{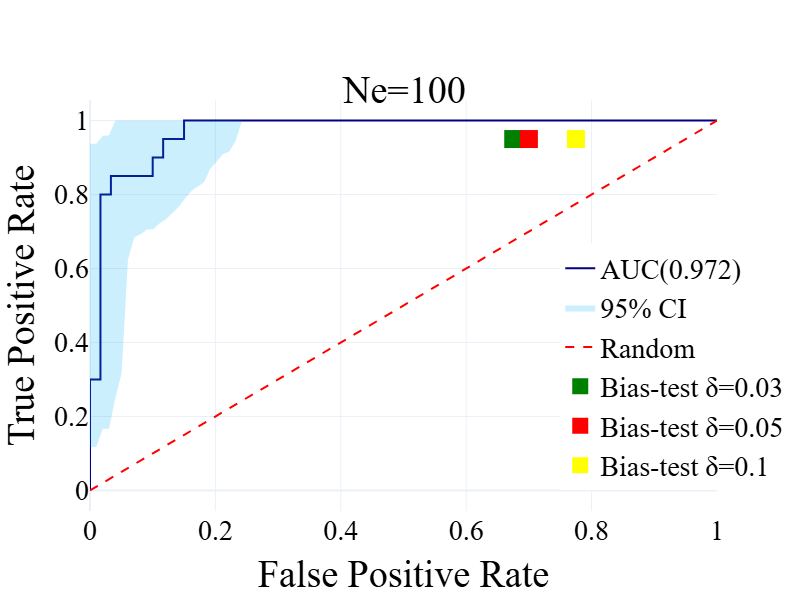}&
        \includegraphics[width =0.3\columnwidth]{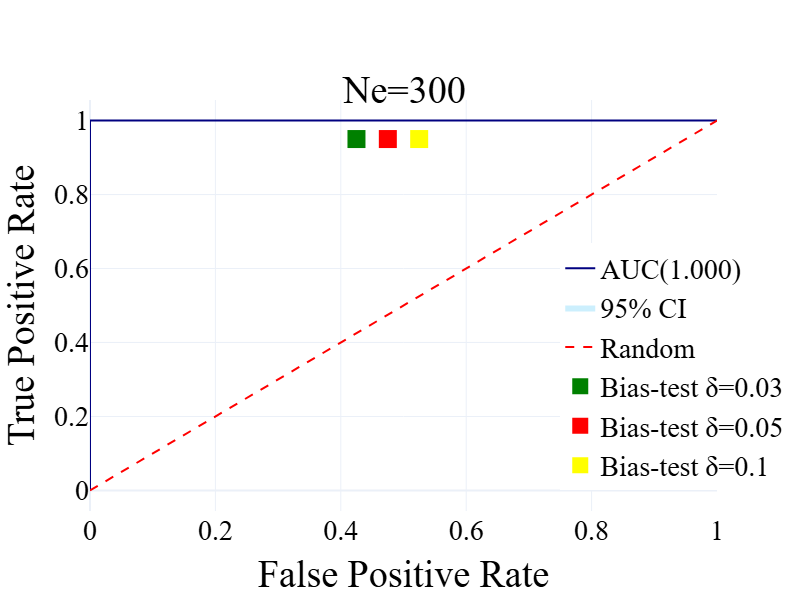}\\
    \end{tabular}
    \caption{\label{fig:AUC_TPR_FPR_Compare}
    \textbf{Synthetic data}: Areas under the ROC curve for predicting $h_{ \textbf{Z}}$ with $N_o^* = 5000$ and an increasing number of $N_e$ samples. Dots in different colors correspond to the false positive and true positive rate of testing s-admissible backdoor sets using \texttt{Bias-test} method.}
    \label{fig:Exp_Case_1}
\end{figure*}

Fig.~\ref{fig:TPR_FPR_Compare} presents the TPR and FPR obtained for different combinations of $N_o^*$ and $N_e$, thereby examining the effect of increasing sample sizes on the performance of both methods. Specifically, results are shown for $N_o = 5000$ with $N_e \in \{50, 100, 300, 1000\}$ (Fig.~\ref{fig:TPR_FPR_Compare}(a)) and for $N_o = 15000$ with $N_e \in \{500, 1000, 3000, 10000\}$ (Fig.~\ref{fig:TPR_FPR_Compare}(b)). Fig.~\ref{fig:TPR_FPR_Compare} shows that \texttt{FindsABS} correctly identifies sets that are sABSs in most cases, even for very small sample sizes, whereas \texttt{Bias-test} tends to produce frequent false positives, particularly for small $N_e$. As both $N_o^*$ and $N_e$ increase, the performance of the two methods converges.

\begin{figure}[h!]
\centering
\begin{tabular}{cc}
\includegraphics[width =0.4\columnwidth]{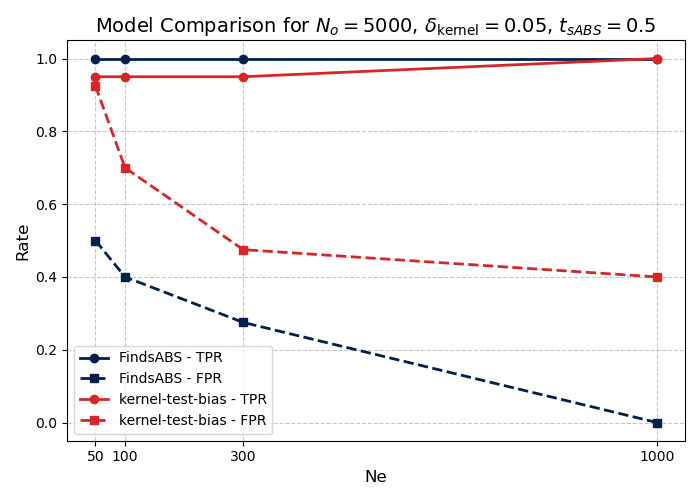}&
\includegraphics[width =0.4\columnwidth]{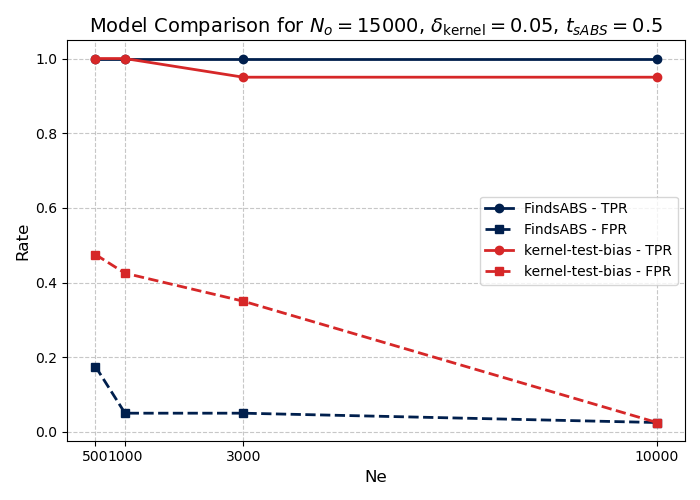}\\ (a) & (b)
\end{tabular}
\caption{\textbf{Synthetic Data:} True Positive and False Positive Rates are compared for \texttt{Bias-test} (red) and \texttt{FindsABS} (blue) using: (a) $N_o^* = 5000$ and $N_e \in \{50, 200, 300, 1000\}$. Both methods correctly identify that the full set (FS) is an sABS. \texttt{FindsABS} achieves a very low FPR even for small experimental sample sizes, whereas \texttt{Bias-test}, particularly for small $N_e$, frequently misidentifies sets as sABS, and (b) $N_o^* = 15000$ and $N_e \in \{500, 1000, 3000, 10000\}$, \texttt{FindsABS} almost always correctly identifies whether a set is an sABS. For \texttt{Bias-test}, false positives still occur at smaller $N_e$, but as the sample size grows, its performance approaches that of \texttt{FindsABS}.
}
\label{fig:TPR_FPR_Compare}
\end{figure}

\textbf{Computational Time}\\
All experiments were conducted on the same workstation. Using the full set of covariates $\{X,Z,W\}$ in experiments of Section \ref{subsec:Compare1} with $N_o^*=5000$ and $N_e=100$, the \texttt{Kernel-bias} method, run for $1000$ epochs with a learning rate of $0.1$ using the authors’ public GitHub code, required 89.37 seconds. \texttt{ProbsABS} (Algorithm~\ref{algo:probsabs}), with $N_s=10000$ sampling iterations and $1000$ warmup steps in the NumPyro MCMC sampler, completed in 27.08 seconds. Finally, the full \texttt{FindsABS} (Algorithm~\ref{algo:FindsABS}), executed using the greedy search procedure, finished in 96.57 seconds.

\subsubsection{Simulation Details and Additional Comparison Using Semi-Synthetic Data}
\label{subsec:Compare2}
As discussed in the Experimental Section~\ref{sec:experiments} of the main paper, we follow the scenarios described in the \textit{Bias Model} paragraph of \citet{de-bartolomeis2024b} and construct semi-synthetic data based on the MineThatData Email dataset \citep{Hillstrom2008}, a randomized marketing experiment involving approximately 64,000 customers. The treatment consisted of receiving either a “Men’s” or “Women’s” email campaign, combined into a single group, while the control group received no email. The outcome of interest was the dollars spent in the two weeks following the campaign. Continuous features were normalized and categorical features one-hot encoded, resulting in 13 covariates, and a constant shift of 30 was added to treated individuals to allow flexible bias introduction.

Let $T$ denotes the treatment variable and
$\mathbf{Z} = \{\text{mens}, \text{womens}, \text{zip\_code}, \text{newbie}, \text{channel}, \text{history}\}$ denotes the set of available covariates. In each scenario, bias ($\delta^*$) is introduced in the outcome for a specific subgroup of individuals. In \textbf{Scenario 1}, a single subgroup receives a constant bias $\delta^\star = 60$, while all other observations remain unbiased. Two application modes are considered:
\[
\delta^* =
\begin{cases}
60, & \text{if\textbf{ Mode 1:} } channel = 1 \text{ and } T = 1,\\
60, & \text{if\textbf{ Mode 2:} } newbie = 1, \; channel = 1, \text{ and } T = 1,\\
0, & \text{otherwise}.
\end{cases}
\]
In \textbf{Scenario 2}, biases of varying magnitudes are introduced across 12 subgroups, defined by combinations of the binary features \texttt{newbie} and \texttt{mens}, and the categorical feature \texttt{channel}.  The largest bias is $\delta^\star = 60$, affecting only 12\% of the observational dataset. The subgroup biases approximately cancel each other on average, resulting in an overall mean bias close to zero, i.e., $\mathbb{E}_{P_{\mathrm{os}}}[\delta^\star(X)] \approx 0.$ For treated units ($T=1$), the bias $\delta^*$ is added according to the following conditions:

\[
\delta^* =
\underbrace{
\begin{cases}
-40, & \text{if channel = 0} \\
-40, & \text{if channel = 1} \\
40, & \text{if channel = 2} 
\end{cases}}_{\text{channel-specific bias}}
+
\underbrace{
\begin{cases}
20, & \text{if mens = 1} \\
0, & \text{otherwise} 
\end{cases}}_{\text{mens bias}}
+
\underbrace{
\begin{cases}
-10, & \text{if newbie = 1} \\
0, & \text{otherwise} 
\end{cases}}_{\text{newbie bias}}
\]

Bias introduced in this way, changes the outcome mechanism across domains, and the selection directly influences the outcome. As a result,\textbf{ no sABS exists}. We applied \texttt{ProbsABS} with $t = 0.5$ and \texttt{Bias-test} with $\delta \in \{20,40,58,70\}$ on the Full Set (FS) of covariates. In the \texttt{Bias-test} method, when the user’s tolerance threshold, $\delta$, is smaller than the maximum bias introduced ($\delta^*=60$), the null hypothesis is expected to be rejected. 

Building on the experiments presented in the main paper for $N_o^* = 5000$ and $N_e \in \{50, 100, 300, 1000\}$, Fig.~\ref{fig:scenarios_combined} reports results for $N_o^* = 15000$ with $N_e \in \{500, 1000, 3000, 5000\}$, evaluating whether the two methods can correctly identify that the Full Set (FS) is not an sABS across different sample sizes. For large biased subgroups (Scenario~1, Mode~1; Scenario~2), \texttt{ProbsABS} rejects more reliably than \texttt{Bias-test} at smaller $N_e$ values when $\delta$ is close to the true induced bias ($\delta^* = 60$). As $\delta$ decreases and using $N_o^* = 15000$, the performance of \texttt{ProbsABS} and \texttt{Bias-test} becomes comparable. For smaller biased subgroups (Scenario~1, Mode~2), \texttt{Bias-test} performs better under stricter tolerance levels, consistent with its design for detecting small-group biases.

\begin{figure}[h!]
    \centering
    \begin{minipage}{0.35\textwidth}
        \centering
        \caption*{\textbf{(a) Scenario 1, Mode 1}}
        \resizebox{\linewidth}{!}{%
        \begin{tabular}{l c c c c c}
        \toprule
         & \texttt{FindsABS} & \multicolumn{4}{c}{\texttt{Bias-test}} \\
        \cmidrule(lr){2-2} \cmidrule(lr){3-5}
        \textbf{$N_e$} & $t=0.5$ & $\delta=20$ & $\delta=40$ & $\delta=58$ & $\delta=70$ \\
        \midrule
        500   & \textbf{20/20} & \textbf{20/20}  & 17/20 & 11/20 & 8/20  \\
        1000  & \textbf{20/20} & \textbf{20/20} & \textbf{20/20} & 11/20 & 8/20  \\
        3000  & \textbf{20/20} & \textbf{20/20} & \textbf{20/20} & 9/20  & 5/20  \\
        5000  & \textbf{20/20} & \textbf{20/20} & \textbf{20/20} & 10/20 & 3/20  \\
        \bottomrule
        \end{tabular}
        }
    \end{minipage}%
    \hspace{10em} 
    \begin{minipage}{0.35\textwidth}
        \centering
        \caption*{\textbf{(b) Scenario 1, Mode 2}}
        \resizebox{\linewidth}{!}{%
        \begin{tabular}{l c c c c c}
        \toprule
         & \texttt{FindsABS} & \multicolumn{4}{c}{\texttt{Bias-test}} \\
        \cmidrule(lr){2-2} \cmidrule(lr){3-5}
        \textbf{$N_e$} & $t=0.5$ & $\delta=20$ & $\delta=40$ & $\delta=58$ & $\delta=70$ \\
        \midrule
        500   & 6/20 & \textbf{19/20} & \textbf{19/20} & 10/20 & 8/20  \\
        1000  & 11/20 & \textbf{20/20} & \textbf{20/20} & 19/20 & 19/20 \\
        3000  & 13/20 & \textbf{20/20}& \textbf{20/20} & \textbf{20/20} & \textbf{20/20} \\
        5000  & 14/20 & \textbf{20/20} & \textbf{20/20} & \textbf{20/20} & \textbf{20/20} \\
        \bottomrule
        \end{tabular}
        }
    \end{minipage}%

    \vspace{0.5em} 

    \begin{minipage}{0.35\textwidth}
        \centering
        \caption*{\textbf{(c) Scenario 2}}
        \resizebox{\linewidth}{!}{%
        \begin{tabular}{l c c c c c}
        \toprule
         & \texttt{FindsABS} & \multicolumn{4}{c}{\texttt{Bias-test}} \\
        \cmidrule(lr){2-2} \cmidrule(lr){3-5}
        \textbf{$N_e$} & $t=0.5$ & $\delta=20$ & $\delta=40$ & $\delta=58$ & $\delta=70$ \\
        \midrule
        500   & \textbf{20/20} & \textbf{20/20} & 14/20 & 14/20 & 11/20 \\
        1000  & \textbf{20/20} & \textbf{20/20} & 15/20 & 11/20 & 9/20  \\
        3000  & \textbf{20/20} & \textbf{20/20} & \textbf{20/20} & 14/20 & 9/20  \\
        5000  & \textbf{20/20} & \textbf{20/20} & \textbf{20/20} & 14/20 & 12/20 \\
        \bottomrule
        \end{tabular}
        }
    \end{minipage}%

    \caption{\textbf{Rejection rates (out of 20 runs) for the hypothesis that the full set is an sABS in the semi-synthetic data.} The full set is not an sABS. \texttt{ProbsABS} correctly rejects the hypothesis when the affected subgroups are large, across all experimental sample sizes (Senario 1, Mode 1; Scenario 2). \texttt{Bias-test} also identifies that the FS is not an sABS especially with smaller $\delta's$. With $N_o^*=15000$ and increasing $N_e$, \texttt{Bias-test} is better in identifying small biased subgroups (Scenario 1, Mode 2).}
    \label{fig:scenarios_combined}
\end{figure}

Using all the available data ($N_e=12800$ and $N_o^*=51200$), both methods in both scenarios correctly identify that the FS is not an SABs.

\section{Proofs of Theorems \ref{the:proof_MB},
\ref{the:proof_equality}, \ref{the:proof_superset}.}

\begin{lemma}\label{lemma:backall}
Let $\vars Z \subseteq \vars V$. Then $\vars Z$ is a backdoor set for $(X,Y)$ in $\graph G$ (with $\graph G = \graph G^*$) or in the selection diagram $\graph D$ if and only if it is a backdoor set in all three.

\begin{proof}
($\Rightarrow$) $\vars Z$ is a backdoor set in $\graph G = \graph G^*$. The graphs $\graph G$ and $\graph D$ differ only in the outgoing edges from $\vars S$ to $\vars V \cup X \cup Y$. These additional edges cannot introduce new backdoor paths between $X$ and $Y$ in $\graph D$. This is because a backdoor path must begin with an edge pointing into $X$, and the only potential new edge into $X$ would come from a variable in $\vars S$. Since variables in $\vars S$ have no incoming edges themselves, they cannot create a new open path from $Y$ to $X$ in \graph D. Thus, $\vars Z$ is a backdoor set in \graph D.\\
($\Leftarrow$) \vars Z is a backdoor set in \graph D. $\graph G (= \graph G^*)$ and \graph D only differ in the outgoing edges from \vars S to $\vars V \cup X\cup Y$. These edges cannot participate in backdoor paths in \graph D, so \vars Z blocks all backdoor paths in $\graph G$ and $\graph G^*$. Thus, $\vars Z$ is a backdoor set in \graph G $=\graph G^*$.

\end{proof}
\end{lemma}
Lemma \ref{lemma:backall} states that, under the assumption of shared causal graphs, $\graph G$, $\graph G^*$, and $\graph D$ have the same backdoor sets. This lemma applies to all proofs in this document.

\begin{lemma}Let $\vars Z \subseteq \vars V$ be an s-admissible backdoor set for $(X, Y)$ relative to a selection diagram $\graph D$ and let $Q \in \vars Z \setminus \mathbf{\mathrm{MB}}(Y)$ that has an m-connecting path $Q \pi_{QY} Y$ with $Y$ given $\vars Z \setminus Q$. Then there exists a variable $W \in \mathbf{\mathrm{MB}}(Y) \setminus \vars Z$ such that: $W \cup \vars Z$ is an s-admissible backdoor set for $(X, Y)$.
\begin{proof}
Let $Q$ be a variable as described above. Then there exists a variable $W \in \mathrm{MB}(Y)$ between $Q$ and $Y$ that is a non-collider on $\pi_{QY}$, otherwise $Q \in \mathrm{Pa}(\mathrm{Dis}(Y))$, and therefore $Q \in \mathrm{MB}(Y)$. Assume also that $W$ is the non-collider that is closest to $Y$. Then $W \notin \vars Z$, otherwise $\pi_{QY}$ would be blocked given $\vars Z \setminus Q$. We will now show, by contradiction, that adding $W$ to the conditioning set $\vars Z$ does not open any backdoor paths from $X$ to $Y$; hence, $\vars Z \cup W$ is a backdoor set. 

\textbf{$\vars W\cup Z$ is also a backdoor set:}\\
Assume that conditioning on $W$ opens a path $\pi_{XY}$ between $X$ and $Y$ that is blocked given just $\vars Z$. Then $W$ must be a descendant of one or more colliders on that path. Let $C$ be the collider closest to $X$ on $\pi_{XY}$ such that $C$ is blocked on $\pi_{XY}$ given $\vars Z$, but open given $\vars Z \cup W$. Then $\pi_{XC}$ is open given $\vars Z$, and $W$ is a descendant of $C$. Let $\pi_{CW}$ be the (possibly empty) directed path from $C$ to $W$, and let $\pi_{WY}$ be the subpath of $\pi_{CY}$ from $W$ to $Y$. Since $C$ is blocked on $\pi_{XY}$ given $\vars Z$, no variable on $\pi_{CW}$ can be in $\vars Z$. 

We now show that  $\pi_{XC} \pi_{CW} \pi_{WY}$ is an open path from $X$ to $Y$ given $\vars Z$ in $D_{\underline{X}}$, which is a contradiction:
Firstly, notice that all subpaths are open given \vars Z, and \var C is a non-collider on $\pi_{XC} \pi_{CW} \pi_{WY}$ ($\pi_{CW}$ is out of $\var C$).  Then $\pi_{CW} \pi_{WY}$ is also open given \vars Z:
\begin{itemize}
    \item \textbf{Case 1a} If $\pi_{WY}$ is out of \var W, then \var W is a non-collider on $\pi_{XC} \pi_{CW} \pi_{WY}$ which is then open.
    \item\textbf{Case 1b} If $\pi_{WY}$ is into \var W, then $\pi_{QW}$ is out of \var W (\var W is a non-collider on that path), which means that \var W is an ancestor of some collider on the path $\pi_{QY}$  (or an ancestor of \var Q, if there are no colliders on $\pi_{QY}$), since $\pi_{QW}$ is open \vars Z. 
\end{itemize}  

Hence, \var W is either a non-collider on the path $\pi_{XC} \pi_{CW} \pi_{WY}$, or a collider and an ancestor of \vars Z. In both cases,  the path $\pi_{XC} \pi_{CW} \pi_{WY}$ is open given \vars Z.  This is a contradiction, since $\vars Z$ is a backdoor set. Thus, $W$ does not open any backdoor paths, and $\vars Z \cup W$ is also a backdoor set.\\

\textbf{$\vars W\cup Z$ is also an s-admissible set:}\\
We will now show that conditioning on \var W cannot open an s-admissible path from \var S to \var Y. To show this, we will show that any such path would already be open given \vars Z.

Assume conditioning on $W$ opens a path $\pi_{SY}$ between \var S and \var Y that is blocked given $\mathbf{Z}$. Then $W$ must be a descendant of one or more colliders on that path. Let $C'$ be the collider closest
to $S$ on $\pi_{SY}$ such that $C'$ is blocked on $\pi_{SY}$ given $\mathbf{Z}$ and open given $\vars Z \cup W$. Then $\pi_{SC'}$ is open given $\vars Z$, and $W$ is a descendant of $C'$. Hence, there exists a (possibly empty) directed path  $\pi_{C'W}$ from $\var C'$ to $\var W$. Moreover, the path $\pi_{SY}$ can be split in two paths: $\pi_{SC'}$ and $\pi_{C'Y}$. 

\textbf{Case 2a: $C'$ is not  on  $\pi_{WY}$.}
    In that case, let $ {\pi_{W Y}}$ be the subpath of $\pi_{C'Y}$ from $W$ to $Y$. Since $C'$ is blocked on $\pi_{SY}$ given $\mathbf{Z}$, no variable on $\pi_{C'W}$ can be in $\mathbf{Z}$. But then  $\pi_{S C'}\pi_{C' W}\pi_{W Y}$ is an open path from S to Y given $\mathbf{Z}$, following the same reasoning as in Case 1a above. Contradiction, since $\mathbf{Z}$ is an s-admissible set.
    
\textbf{Case 2b: $C'$ is on  $\pi_{WY}$.}
     In that case,  $\pi_{WY}$ can be split in two subpaths:  $\pi_{WC'}$ and $\pi_{C'Y}$. $\pi_{C'Y}$  is open given \vars Z (as part of the open path $\pi_{WY}$).
     \begin{itemize}
         \item If $\pi_{C'Y}$ has a tail into  $\var C'$, then $\var C'=W$ is a non-collider on the path $\pi_{SC'}\pi_{C'Y}$, which means that this path is open given \vars Z. \footnote{\var W has to be $\var C'$ in this case,since \var W is the first non-collider on $\pi_{QY}$}.
         \item If $\pi_{C'Y}$ path has an arrowhead into $\var C'$, then $\var C'$ would be a collider in the path $\pi_{SC'}\pi_{C'Y}$. In addition, $C'$ would also be a collider on $\pi_{QY}$ (since 
         \var W is non-collider on $\pi_{Q Y}$ that is closest to \var Y.) But then  $\var C'$ is an ancestor of some variable in \vars Z, since $\var C'$ is open on the path $\pi_{QY}$. Hence, $\pi_{SC'}\pi_{C'Y}$ is also open given \vars Z.
     \end{itemize}
     
We have shown that in all cases, if \var W was to open an s-admissible or backdoor path, this path would be open given \vars Z. Hence conditioning on \var W cannot open an s-admissible path that was blocked given \vars Z, so $\vars Z\cup W$ is an s-admissible backdoor set.
\end{proof}
\end{lemma}

\begin{repeatedtheorem} [\ref{the:proof_MB}]
If there exists a set $\vars Z \subseteq \vars V$ that is an s-admissible backdoor set for $(X,Y)$ relative to a selection diagram $D$, there always exists a set $\mathbf{Z^*}\subseteq MB(Y)$ that is also s-admissible backdoor set in \graph D.
\end{repeatedtheorem}
\begin{proof}
Let $\mathbf{Z} = \mathbf{W} \cup \mathbf{Q}$ consist of two disjoint sets of variables, where $\mathbf{W} \subseteq \mathrm{MB}(Y)$ and $\mathbf{Q} \notin \mathrm{MB}(Y)$, and suppose that $\mathbf{Z}$ is an sABS.  

To prove the theorem, we define a constructive procedure that iteratively applies \textit{Lemma~2}. After this step, reduces the set $\mathbf{Z}$ by eliminating a variable $Q \in \mathbf{Z} \setminus \mathrm{MB}(Y)$, ensuring that the process terminates after finitely many steps.

\begin{algorithm}[t]
\LinesNumbered
\ForEach{$Q \in \mathbf{Q}$}{
    \ForEach{m-connecting path $\pi_{QY}$ from $Q$ to $Y$ given $\mathbf{Z}$}{
        find $W$ such that Lemma~2 holds\;
        $\mathbf{Z} \gets \mathbf{Z} \cup W$
        \tcc*[r]{$\mathbf{Z}$ remains an sABS}
    }
    $\mathbf{Z} \gets \mathbf{Z} \setminus Q$
    \tcc*{$\mathbf{Z}$ remains an sABS}
}

\caption{Construction of an sABS: $\mathbf{Z}^* \subseteq MB(Y)$ using Lemma~2}\label{algo:sabs}
\end{algorithm}

The inner \texttt{while}-loop (lines 2--5) terminates once there is no longer an m-connecting path $\pi_{QY}$ between $Q$ and $Y$ given $\mathbf{Z} \setminus Q$.  

Line~6 removes a variable $Q \in \mathbf Q$, if $Q$ does not have an m-connecting path with $Y$ given $\mathbf{Z} \setminus Q $. Removing such a $Q$ cannot introduce new paths from $X$ to $Y$ or from $\mathbf{S}$ to $Y$, since all paths from $Q$ to $Y$ are already blocked given $\mathbf{Z} \setminus Q$.  

The algorithm terminates when there is no variable $Q \in \mathbf{Z} \setminus \mathrm{MB}(Y)$ that has an m-connecting path with $Y$ given $\mathbf{Z} \setminus Q$. At that point, the final conditioning set satisfies
\[
\mathbf{Z}^* = \mathbf{W} \subseteq \mathrm{MB}(Y).
\]
In the most exhaustive case, the procedure yields $\mathbf{Z}^* = \mathrm{MB}(Y)$, which makes all other variables independent of $Y$ conditional on $\mathbf{Z}^*$. Thus, the construction ensures that $\mathbf{Z}^*$ is a valid sABS.

\end{proof}

\begin{definition}[Conditional Entropy] \label{def:condentropy}
\noindent Let $P$ be the full joint probability distribution over a set of variables $\boldsymbol{V}$, let $Y\in \boldsymbol{V}$ be a variable, and let $\vars Z \subseteq \boldsymbol{V} \setminus \{Y\}$ be a set of  variables. Then, the conditional entropy of $Y$ given $\vars Z$ is defined as follows \citep{cover1999elements}:
\begin{equation}\label{eq:entDisc}
        H(Y|\vars Z) = - \sum_{y}{\sum_{z}{P(y,z)\cdot \log P(y|z)}}
\end{equation}
where $y$ and $z$ denote the values of $Y$ and $\vars Z$, respectively. 
\end{definition}

\begin{lemma}\label{lem:ent}
Let $X,Y \in \boldsymbol{V}$ be two variables and $\vars Z  \subseteq \boldsymbol{V} \setminus \{X,Y\}$ be a set of  variables. Then, $H(Y|\vars Z) \geq H(Y|X,\vars Z)$, where the entropies are defined by Definition \ref{def:condentropy}, and the equality holds if and only if $Y \CI X | \vars Z$.
\end{lemma}
\begin{proof}
Applying the chain rule of entropy, the conditional mutual information can be computed as follows~\cite{cover1999elements}:
\begin{equation}
        I(X; Y|\vars Z) = H(Y|\vars Z) - H(Y|X,\vars Z)  \,.
\end{equation}
Given that the mutual information is nonnegative (i.e., $I(X; Y|\vars Z) \geq 0$) and $I(X; Y|\vars Z) = 0$ if and only if $Y \CI X |  \vars Z$ (see~\citet{cover1999elements}, page 29), it follows that:
\begin{equation}
\begin{split}
        H(Y|\vars Z) - H(Y|X,\vars Z) \geq 0\\
        H(Y|\vars Z)  \geq H(Y|X,\vars Z) 
         \,,
\end{split}
\end{equation}
where the equality holds if and only if $Y\CI X | \vars Z$. 
\end{proof}


\begin{lemma} \label{lem:logPZ_CMB}
Given Assumptions \ref{ass:main}, \ref{ass:sfaith}, the BD score for $\log P(D_e|D_o^*, h_{\vars Z})$ in the large sample limit is defined as follows:
\begin{equation}\label{eq:CBM}
\begin{split}
\lim_{N \rightarrow \infty} & \log P(D_e|h_\vars{Z}, D_o^*) =\\
& \lim_{N \rightarrow \infty} -(N_o+N_e) \cdot H_{o,e}(Y|X, \vars Z) + N_o \cdot H_{o}(Y|X, \vars Z) \\ & -\frac{q(r - 1)}{2} \left[\log (N_o+N_e) -\log N_o\right] + const.
\end{split}
\end{equation}
Additionally, the BD score for $\lim_{N \rightarrow \infty} \log P(D_e|D_o^*, \neg h_\vars{Z})$ is defined as follows:
\begin{equation}\label{eq:nCBM}
    \begin{split}
    \lim_{N \rightarrow \infty} & \log P(D_e|D_o^*, \neg h_\vars Z) =\\
    & \lim_{N \rightarrow \infty} -N_e \cdot H_{e}(Y|X, \vars Z) -\frac{q(r - 1)}{2} \log N_e + const.
\end{split}
\end{equation}
\end{lemma}
\begin{proof} The proofs are similar to the proofs of   Lemmas 2.5, 2.6 in \citet{pmlr-v206-triantafillou23a}.
\end{proof}

\begin{lemma}\label{lem:Hoe}
Let $P_{o^*}, P_e, P_{o^*,e}$ denote the joint probability distribution in the observational, experimental, and joint data, respectively. Also, let $\vars Z \subseteq \vars O$ be a subset of variables. Then,
\begin{equation}\label{eq:Hoe}
    2H(P_{o^*, e}(Y|X, \vars Z))\geq H(P_{o^*}(Y|X, \vars Z))+H(P_e(Y|X, \vars Z)),
\end{equation}
where the equality in Equation (\ref{eq:Hoe}) holds if and only if Eq. \ref{eq:ident} holds.
\end{lemma}
\begin{proof} The proof is similar the proof of Lemma 2.7 in \citet{pmlr-v206-triantafillou23a}.
\end{proof}

To simplify the notation, we use  $H_{o^*, e}(Y|X, \vars Z), H_{o^*}(Y|X, \vars Z), H_{e}(Y|X, \vars Z)$ to denote $H(P_{o^*, e}(Y|X, \vars Z)),$$ H(P_{o^*}(Y|X, \vars Z))$, and $ H(P_{e}(Y|X, \vars Z))$, respectively.

\begin{repeatedtheorem}[\ref{the:proof_equality}]
Assume Assumptions \ref{ass:main}, \ref{ass:sfaith} hold, $X, Y, \vars O$ are discrete, and $ N_o$ and $N_e$ increase equally without limit ($N:=N_e=N_o$ in the limit). Then Eq. \ref{eq:hzt} will converge to 1 if and only if \vars Z is an s-admissible backdoor set.
\begin{equation}
    \begin{cases}
    \displaystyle \lim_{N \rightarrow \infty}P(h_\vars{Z}|D_e, D_o^*) = 1, & \text{\vars Z is an sABS}   \\
    \smallskip
    \displaystyle \lim_{N \rightarrow \infty} P(h_\vars{Z}|D_e, D_o^*) = 0,& \text{ otherwise}
    \end{cases}
\end{equation}
\end{repeatedtheorem}
\begin{proof}
 
\noindent For a set $\vars Z$, we have that
\begin{equation}\label{eq:approx1}
\begin{split}
 \lim_{N \rightarrow \infty} P(h_{\vars Z}|D_o^*, D_e) \lim_{N \rightarrow \infty} =\frac{P(D_e|D_o^*, h_{\vars Z})P(h_{\vars Z}|D_o^*)}{P(D_e|D_o^*, h_{\vars Z})P(h_{\vars Z}|D_o^*)+P(D_e|D_o^*, P(\neg h_{\vars Z}|D_o^*)} .   
\end{split}
\end{equation}

By inverting Equation (\ref{eq:approx1}),  and for $P(h_{\vars Z}|D_o^*) = 1/2$ we obtain the following:
\begin{equation}\label{eq:inv1}
\begin{split}
  \lim_{N \rightarrow \infty} \frac{1}{P(h_{\vars Z}|D_o^*, D_e)} = 
  \lim_{N \rightarrow \infty} \frac{P(D_e|D_o^*, h_{\vars Z})+P(D_e|D_o^*, \neg h_{\vars Z})}{P(D_e|D_o^*, h_{\vars Z})} &=\\
  1+\lim_{N \rightarrow \infty}( \frac{P(D_e|D_o^*, \neg h_{\vars Z})}{P(D_e|D_o^*, h_{\vars Z})})&=\\
  1+
   \lim_{N \rightarrow \infty} \exp (\log{\frac{P(D_e|D_o^*, \neg h_{\vars Z})}{P(D_e|D_o^*, h_{\vars Z})}})
\end{split}
\end{equation}

\noindent Using Equations (\ref{eq:CBM}) and (\ref{eq:nCBM}), we obtain $log (\frac{P(D_e|D_o^*, \neg h_{\vars Z})}{P(D_e|D_o^*, h_{\vars Z})})$ in the large sample limit as follows:

\begin{equation}\label{eq:logf2}
\begin{split}
    & \lim_{N \rightarrow \infty} log (\frac{P(D_e|D_o^*, \neg h_{\vars Z})}{P(D_e|D_o^*, h_{\vars Z})})= \lim_{N \rightarrow \infty}\log {P(D_e|D_o^*, \neg h_{\vars Z})} - \lim_{N \rightarrow \infty} \log{P(D_e|D_o^*, h_{\vars Z})} \\
    &= \lim_{N \rightarrow \infty} -N_e \cdot H_{e}(Y|X, \vars Z) + (N_o+N_e) \cdot H_{o^*, e}(Y|X, \vars Z) - N_o \cdot H_{o^*}(Y|X, \vars Z)\\
    & -\frac{q(r - 1)}{2} \log N_e + \frac{q(r - 1)}{2} \left[\log (N_o+N_e) -\log N_o\right] + const.\\
    &= \lim_{N \rightarrow \infty} N \cdot [- H_{e}(Y|X, \vars Z)+ 2  H_{o^*, e}(Y|X, \vars Z) - H_{o^*}(Y|X, \vars Z)]\\
    & -\frac{(r - 1)}{2} (q \log N - q\log 2) + const\\
    &=\lim_{N \rightarrow \infty} N \cdot [- H_{e}(Y|X, \vars Z) + 2  H_{o^*, e}(Y|X, \vars Z) - H_{o^*}(Y|X, \vars Z)]\\
    & -\frac{q(r - 1)}{2} (\log {\frac{N}{2}}) + const.
    \end{split}
\end{equation}
where the last step is possible since  $N_e = N_o \coloneqq N$.
\vspace{5mm}

\noindent If $\vars Z$ is an sABS set, it follows  from Lemma \ref{lem:Hoe} that
$$lim_{N\rightarrow \infty}H_{o, e}(Y|X, \vars Z)=lim_{N\rightarrow \infty}H_{o}(Y|X, \vars Z)=lim_{N\rightarrow \infty}H_{e}(Y|X, \vars Z);$$
therefore
\begin{equation}
    \begin{split}
        & \lim_{N \rightarrow \infty} log (\frac{P(D_e|D_o^*, \neg h_{\vars Z}}{P(D_e|D_o^*, h_{\vars Z})})=\lim_{N \rightarrow \infty}  -\frac{q(r - 1)}{2} (\log {\frac{N}{2}}) + const = -\infty
    \end{split}
\end{equation}

Hence by Eq. \ref{eq:inv1},   \[\lim_{N \rightarrow \infty} \frac{1}{P(h_{\vars Z}|D_o^*, D_e)}\rightarrow 1\]
and therefore $P(h_{\vars Z}|D_o^*, D_e)$ goes to 1 as $N$ goes to infinity.

If $\vars Z$ is not an sABS, then by Lemma \ref{lem:Hoe}, when $N\rightarrow \infty$
$$- H_{e}(Y|X,\vars Z) + 2  H_{o^*, e}(Y|X,\vars Z) - H_{o^*}(Y|X,\vars Z)>0$$ and therefore $$\lim_{N \rightarrow \infty} N \cdot [- H_{e}(Y|X,\vars Z)+ 2  H_{o^*, e}(Y|X,\vars Z) - H_{o^*}(Y|X,\vars Z)] =\infty.$$
Notice that this term is $O(N)$ and will dominate the second term, $-\frac{q(r-1)}{2} \log{\frac{N}{2}}$. Therefore 
\begin{equation}
    \begin{split}
        & \lim_{N \rightarrow \infty} log (\frac{P(D_e|D_o^*, \neg h_{\vars Z}}{P(D_e|D_o^*, h_{\vars Z})})=\infty,
    \end{split}
\end{equation}
and by Eq. \ref{eq:approx1}  $$\lim_{N \rightarrow \infty} \frac{1}{P(h_{\vars Z}|D_o^*, D_e)} = \infty,$$ thus  $P(h_{\vars Z}|D_o^*, D_e)$ goes to 0 as $N$ goes to infinity.

\end{proof}

\begin{repeatedtheorem}[\ref{the:proof_superset}]
Assume Assumptions \ref{ass:main}, \ref{ass:sfaith}  hold, $X, Y, \vars O$ are discrete, and $ N_o$ and $N_e$ increase equally without limit ($N:=N_e=N_o$ in the limit). Let \vars Z, \vars Z' be s-admissible backdoor sets, $\vars Z\subset \vars Z'$, and \( (Y \perp\!\!\!\perp \vars Z'\setminus Z \mid \vars  Z)_{D_{\overline{X}}} \). Then, 
$$\lim_{N \rightarrow \infty}P(D_e|h_{\vars Z}, D_o^*)> \lim_{N \rightarrow \infty}P(D_e|h_{\vars Z'}, D_o^*)$$
\end{repeatedtheorem}
\begin{proof} Since both \vars Z, \vars Z' are s-admissible backdoor sets, the following hold:
\begin{equation}
    P(Y|do(X), \vars Z, \vars s) = P(Y|X, \vars Z, \vars s^*)
,\qquad
    P(Y|do(X), \vars Z', \vars s) = P(Y|X, \vars Z', \vars s^*)
\end{equation}
Moreover,  since   \( (Y \perp\!\!\!\perp \vars Z'\setminus \vars Z \mid \vars  Z)_{D_{\overline{X}}} \) 
\begin{equation}
    P(Y|do(X), \vars Z', \vars s^*) = P(Y|do(X), \vars Z, \vars s^*) 
\end{equation}
Hence, in the limit, the entropies in Eq. \ref{eq:CBM} are the same for \vars Z, \vars Z':

\[\lim_{N \rightarrow \infty}H_{o^*,e}(Y|X, \vars Z)= \lim_{N \rightarrow \infty}H_{o^*}(Y|X, \vars Z) = \lim_{N \rightarrow \infty}H_{e}(Y|X, \vars Z')  = \lim_{N \rightarrow \infty}H_{e}(Y|X, \vars Z)\]

\begin{equation}\label{eq:logf1}
\begin{split}
     \lim_{N \rightarrow \infty} &\log P(D_e|D_o^*, h_{\vars Z}) - \lim_{N \rightarrow \infty} \log P(D_e|D_o^*, h_{\vars Z'}) \\
    &= \lim_{N \rightarrow \infty} -(N_o+N_e) \cdot H_{o,e}(Y|X, \vars Z) + N_o \cdot H_{o}(Y|X, \vars Z) \quad - \frac{q(r - 1)}{2} \left[\log (N_o+N_e) - \log N_o\right] \\
    & + (N_o+N_e) \cdot H_{o,e}(Y|X,\vars Z') - N_o \cdot H_{o}(Y|X, \vars Z') 
     + \frac{q'(r - 1)}{2} \left[\log (N_o+N_e) - \log N_o\right]  \\
    &= \lim_{N \rightarrow \infty}(N_o+N_e) \cdot \left[H_{o,e}(Y|X, \vars Z') - H_{o,e}(Y|X, \vars Z)\right] \quad + N_o \cdot \left[H_{o}(Y|X, \vars Z) - H_{o}(Y|X, \vars Z') \right] \\
    & - \frac{(q-q') (r - 1)}{2} \left[\log (N_o+N_e) - \log N_o\right]  \\
    &= \lim_{N \rightarrow \infty} 
     - \frac{(q-q') (r - 1)}{2} \log 2.
\end{split}
\end{equation}
where the last step is possible since both $N_e = N_o \coloneqq N$. Since $q'>q$ and $r>1$ Eq. \ref{eq:logf1}$>0$.
\end{proof}

\end{document}